\documentclass[12pt,a4paper]{article}
\pdfoutput=1

\usepackage[margin=25mm]{geometry}
\usepackage{amsmath,amssymb,mathtools}
\usepackage{bm}
\DeclareMathOperator*{\median}{median}

\usepackage{graphicx}
\usepackage{booktabs}
\usepackage{caption}
\usepackage{subcaption}
\usepackage{algorithm}
\usepackage{algorithmic}
\usepackage{amsthm}
\usepackage{aliascnt}
\usepackage{authblk}

\usepackage{newtxtext}
\usepackage{newtxmath}

\usepackage{hyperref}
\hypersetup{
  colorlinks=true,
  linkcolor=blue,
  citecolor=blue,
  urlcolor=blue
}
\usepackage{cite}
\usepackage{cleveref}
\providecommand{\orcidlink}[1]{\href{https://orcid.org/#1}{[ORCID]}}
\IfFileExists{orcidlink.sty}{\let\orcidlink\relax\usepackage{orcidlink}}{}

\makeatletter
\def\@fnsymbol#1{\ifcase#1\or 1\or 2\or *\or \textdagger\or
  \textdaggerdbl\or \textsection\or \textparagraph\or \|\or
  **\or \textdagger\textdagger\or \textdaggerdbl\textdaggerdbl\else
  \@ctrerr\fi\relax}
\makeatother

\DeclareMathOperator*{\argmin}{arg\,min}

\providecommand{\keywords}[1]{\medskip\noindent\textbf{Keywords:} #1}

\theoremstyle{plain}
\newtheorem{theorem}{Theorem}[section]

\newaliascnt{proposition}{theorem}
\newtheorem{proposition}[proposition]{Proposition}
\aliascntresetthe{proposition}

\newaliascnt{corollary}{theorem}
\newtheorem{corollary}[corollary]{Corollary}
\aliascntresetthe{corollary}

\newaliascnt{lemma}{theorem}
\newtheorem{lemma}[lemma]{Lemma}
\aliascntresetthe{lemma}

\crefname{theorem}{Theorem}{Theorems}
\Crefname{theorem}{Theorem}{Theorems}
\crefname{proposition}{Proposition}{Propositions}
\Crefname{proposition}{Proposition}{Propositions}
\crefname{lemma}{Lemma}{Lemmas}
\Crefname{lemma}{Lemma}{Lemmas}
\crefname{corollary}{Corollary}{Corollaries}
\Crefname{corollary}{Corollary}{Corollaries}

\title{
  Analytical Evaluation of DCA Convergence Properties for Minimizing Prediction Functions of Gaussian RBF Support Vector Regression
}
\author[1]{Yohei Kakimoto}
\author[1]{Yuto Omae}
\author[2]{Hirotaka Takahashi}
\affil[1]{Nihon University}
\affil[2]{Tokyo City University}

\date{}

\begin{document}

\maketitle

\begin{abstract}
For nonconvex optimization problems whose objective is the prediction function of a trained Support Vector Regression (SVR) model with the Gaussian radial basis function (RBF) kernel (RBF-SVR), we present a framework that applies the difference of convex functions (DC) algorithm (DCA) by exploiting the analytical structure of the RBF kernel to construct an explicit DC decomposition.
Specifically, we derive in closed form both the lower bound $\mu$ of the strong convexity parameter of the DC components and the upper bound $L$ of the gradient Lipschitz constant of the subproblem.
Both $\mu$ and $L$ are determined solely by the post-training dual-coefficient sum $C_{\alpha}$ and the RBF kernel parameter $\gamma$, together with the DC decomposition parameter $\rho$, and they share a common leading term $C_{\alpha}\rho$.
Through numerical experiments on six benchmark functions, we show that $C_{\alpha}\rho$ is the primary single quantity characterizing both the convergence properties and the initial-point dependence of DCA, and further demonstrate that it decomposes into two independent pathways, $C \to C_{\alpha}$ and $\gamma \to \rho$, with its primary variation governed by the SVR hyperparameters $(C, \gamma)$.
Together, these results allow the convergence properties of DCA on RBF-SVR to be assessed in advance through the single scalar quantity $C_{\alpha}\rho$: approximately from $(C, \gamma)$ before training, and exactly in closed form after training.
\end{abstract}


%%%%%%%%%%%%%%%%%%%%%%%%%%%%%%%%%%%%%%%%%%
\section{Introduction}\label{sec:introduction}
In formulating an optimization problem, the choice of objective function substantially affects the quality of the solution and the computational efficiency.
Much of the conventional optimization research has assumed analytically given objective functions, such as linear or convex quadratic functions.
However, in real-world problems, the objective function can be difficult to define analytically.
When the optimization target is a response that can only be evaluated through simulation, experiment, or observation, expressing the objective function in closed form is generally difficult.
In such cases, a common approach is to construct a regression model from observed data using machine learning, and to optimize the resulting prediction function as a surrogate for the objective function~\cite{Simpson2001,Queipo2005,WangShan2007,Forrester2009,ShanWang2010,Viana2014}.

In general, the choice of regression model determines the class of the optimization problem.
For example, with a linear model, the objective function is linear and the optimization is tractable, but the model's ability to approximate complex phenomena is limited.
On the other hand, with a nonlinear regression model, the approximation accuracy improves, but the prediction function is nonlinear, and the optimization problem becomes a nonlinear program.
A variety of nonlinear regression models are used as surrogates, including neural networks, Gaussian process regression, decision-tree-based methods, and kernel-based methods~\cite{Hornik1989,Sacks1989,Vapnik2000,Friedman2001,Smola2004,Chen2016,Raissi2019}.
While these models offer high approximation capability, their prediction functions are generally nonconvex, and the resulting optimization problem belongs to an even more computationally challenging class.

Among these, we focus on Support Vector Regression (SVR)~\cite{Smola2004} with the Gaussian radial basis function (RBF) kernel.
Hereinafter, we refer to SVR with the Gaussian RBF kernel simply as RBF-SVR.
RBF-SVR has flexible approximation capability and a strong theoretical foundation, and is widely used as a surrogate~\cite{ShanWang2010,Viana2014}.
Furthermore, the prediction function is expressed explicitly as a linear combination of kernel functions, allowing its algebraic structure to be treated analytically.

Traditionally, when optimizing the RBF-SVR prediction function as an objective, metaheuristics such as genetic algorithms and particle swarm optimization have been widely used.
These methods treat the prediction function as a black box, and therefore do not exploit the algebraic structure of the RBF-SVR prediction function.
In addition, since they are inherently approximate, convergence even to a local optimum is generally not guaranteed.
On the other hand, recent years have seen progress in optimization methods that explicitly exploit the difference of convex functions (DC) structure of the prediction functions of trained machine learning models~\cite{Awasthi2024,Maskan2025,Pokutta2025}.
The DC algorithm (DCA), based on DC programming~\cite{PhamDinh1997,PhamDinh1998,LeThi2018}, is a local optimization method that searches for an optimal solution by expressing a nonconvex function as the difference of two convex functions and iteratively solving the resulting convex subproblems.
Since the subproblems are convex, convex constraints can also be handled naturally.
More recently, methods based on the Frank-Wolfe method have been proposed for DC optimization, making it possible to handle constrained DC optimization problems at scale~\cite{Maskan2025,Pokutta2025}.
Hereinafter, we refer to these Frank-Wolfe-based DC optimization methods collectively as DC-FW.
Since DC-FW takes the smoothness constant as a given input in its convergence analysis, determining this constant for the specific target function is an important practical challenge.

Motivated by these considerations, in this paper we construct a DC decomposition for the RBF-SVR prediction function and analyze the structure of the resulting DCA subproblem.
The Hessian of the RBF kernel can be expressed as the sum of a scalar multiple of the identity matrix and a rank-one matrix, and its eigenvalues can be bounded in closed form via the Rayleigh quotient.
Through this structure, the parameters of the DC decomposition, along with the strong convexity and smoothness parameters of the DCA subproblem, are analytically determined from the kernel parameters and the dual coefficients obtained after training the SVR, before the DCA iterations begin.
Furthermore, since the primary variation of these parameters is governed by the SVR hyperparameters such as $C$ and $\gamma$, they can also be approximately estimated at the design stage before training.
These quantities provide analytical reference values for the smoothness constant required by DC-FW~\cite{Maskan2025,Pokutta2025} when it is applied to the minimization of the RBF-SVR prediction function.

The contributions of this paper are as follows.
First, we explicitly derive a DC decomposition for the RBF-SVR prediction function, present in closed form the convexity condition on the DC decomposition parameter $\rho$, and provide the corresponding DCA formulation.
Second, we show that the lower bound $\mu$ of the strong convexity parameter and the upper bound $L$ of the gradient Lipschitz constant of the DCA subproblem are analytically determined in closed form from the SVR hyperparameters and the post-training dual coefficients.
Specifically, both $\mu$ and $L$ are determined solely by the sum of SVR dual coefficients $C_{\alpha}$, the RBF kernel parameter $\gamma$, and the DC decomposition parameter $\rho$, and they share the common leading term $C_{\alpha}\rho$.
In particular, $L$ provides an analytical reference value for the smoothness constant required by DC-FW~\cite{Maskan2025,Pokutta2025} when it is applied to the minimization of the RBF-SVR prediction function.
Third, we show through numerical experiments on six benchmark functions that $C_{\alpha}\rho$ is the primary single quantity characterizing both the convergence properties and the initial-point dependence of DCA.
Fourth, we experimentally show that $C_{\alpha}\rho$ decomposes into two independent pathways, $C \to C_{\alpha}$ and $\gamma \to \rho$, and that its primary variation is governed by the SVR hyperparameters $(C, \gamma)$.
Through these analyses, we reveal a two-level structure: the convergence properties of DCA are summarized by the single quantity $C_{\alpha}\rho$, whose primary variation is in turn governed by the pre-training hyperparameters.
The novelty of this paper lies in directly exploiting the algebraic structure of the RBF-SVR prediction function, thereby filling open gaps in two lines of prior research: optimization using SVR as a surrogate and DC applications to trained machine learning models.
Detailed comparisons with prior work are presented in~\cref{sec:related-work}.

Our results make it possible to estimate the convergence properties of DCA through the choice of hyperparameters, at the design stage before training the RBF-SVR.
Furthermore, for a trained RBF-SVR model, the convergence properties can be analytically evaluated before the DCA iterations begin.
These results position this work as an analytical complement, for a specific target function, to the line of research on optimization that exploits DC structure.

%%%%%%%%%%%%%%%%%%%%%%%%%%%%%%%%%%%%%%%%%%
\section{Related Work}\label{sec:related-work}
Optimization studies that use SVR as a surrogate have been reported across various fields of engineering design~\cite{Nakayama2003,Yun2009,Pan2010,Zhu2012,Andres2012,Lee2008,Zhai2023}.
In particular, Nakayama et al. have systematically developed a framework for sequential approximate multi-objective optimization that uses RBF-SVR as a surrogate~\cite{Nakayama2003,Yun2009,Nakayama2009book}.
However, these studies treat the surrogate as a black box and solve it using general-purpose metaheuristics or nonlinear programming solvers.
Consequently, the algebraic structure of the RBF-SVR prediction function is not exploited.
We also note that there are frameworks, such as Bayesian optimization~\cite{Jones1998,Shahriari2016}, that optimize an acquisition function incorporating both the surrogate's prediction and its uncertainty, rather than the prediction alone.
By contrast, this paper directly optimizes the trained RBF-SVR prediction function itself rather than an acquisition function, and exploits its algebraic structure through DC decomposition.

DC programming is an optimization framework that expresses a nonconvex function as the difference of two convex functions, and has been systematically developed by Pham Dinh and Le Thi since the 1980s~\cite{PhamDinh1997,LeThi2018}.
The DCA is an iterative method based on this framework that constructs a convergent sequence for a nonconvex problem by solving a convex subproblem at each iteration.
DC programming has wide applications in machine learning, sparse optimization, finance, and other domains, and it is known that many nonconvex optimization problems can be addressed through DC decomposition~\cite{Collobert2006,LeThi2009,Gotoh2018}.

Recent years have seen progress in research that optimizes the prediction functions of trained machine learning models through DC structure.
For multilayer perceptrons and convolutional neural networks with ReLU activations, Awasthi et al. explicitly constructed a layer-wise inductive DC decomposition of the prediction function by decomposing the weight matrices into positive and negative components and expressing ReLU as $\max(\tilde{f}, \tilde{g}) - \tilde{g}$~\cite{Awasthi2024}.
This demonstrates that the computation of adversarial attacks and the approximate optimization of complex functions through trained networks can both be addressed within the DCA framework.
However, their construction fundamentally relies on the piecewise linearity of ReLU, and a direct extension to smooth surrogate models such as the RBF kernel is not straightforward.
Moreover, quantitative metrics that govern the DCA convergence error, such as the strong convexity parameter and Lipschitz constants, are not analyzed.

There is also a line of research on computational methods for DCA that do not depend on a specific target function.
For $L$-smooth nonconvex functions, Maskan et al. introduced a quadratic DC decomposition with $\frac{L}{2}\|\bm{x}\|^2$ as an auxiliary term and proposed DC-FW, which solves the DCA subproblem via the Frank-Wolfe method~\cite{Maskan2025}.
They noted that the choice of DC decomposition affects both the geometry of the subproblem and the quality of the stationary points reached, and showed that an appropriate choice of decomposition can improve computational efficiency.
Pokutta~\cite{Pokutta2025} extended this framework with further computational refinements, combining Blended Pairwise Conditional Gradients with warm-starting and adaptive error bounds to construct an advanced DC-FW capable of handling constrained DC optimization problems at scale.
However, DC-FW treats the smoothness constant $L$ as given input and does not address how to choose $L$ for a specific target function.

This paper focuses on the RBF-SVR prediction function as a specific target function and conducts an analysis that exploits its algebraic structure.
Specifically, we exploit the analytical structure of the RBF kernel to provide a closed-form lower bound on the DC decomposition parameter $\rho$ for individual kernel functions, and construct a DC decomposition for the RBF-SVR prediction function through component decomposition by the sign of the dual coefficients.
Furthermore, we provide in closed form both the lower bound $\mu$ of the strong convexity parameter of the DC components and the upper bound $L$ of the gradient Lipschitz constant of the subproblem.
Both $\mu$ and $L$ share a common leading term $C_{\alpha}\rho$, which can be computed in advance from the SVR hyperparameters and the post-training dual coefficients.
Consequently, for a trained RBF-SVR model, the convergence properties of DCA can be analytically evaluated before the iterations begin.
Moreover, we show that $C_{\alpha}\rho$ can be estimated through the choice of hyperparameters even before training the SVR, thereby enabling approximate evaluation of the DCA convergence properties at the pre-training stage.
In addition, the gradient Lipschitz upper bound $L$ provides an analytical reference value for the smoothness constant when DC-FW~\cite{Maskan2025,Pokutta2025} is applied to the minimization of the RBF-SVR prediction function, thereby positioning this work as an analytical complement to these methods for a specific target function.

%%%%%%%%%%%%%%%%%%%%%%%%%%%%%%%%%%%%%%%%%%
\section{Preliminaries}\label{sec:preliminaries}
%%%%%%%%%%%%%%%%%%%%%%%%%%%%%%%%%%%%%%%%%%
\subsection{SVR Formulation and Prediction Function}\label{sec:svr}
Given training samples $\{(\bm{x}_i, y_i)\}_{i=1}^n$ with $\bm{x}_i \in \mathbb{R}^d$ and $y_i \in \mathbb{R}$, SVR determines the parameters $\bm{w}$ and $b$ of the prediction function
\begin{align*}
  \hat{f}(\bm{x}) = \langle \bm{w}, \bm{x} \rangle + b
\end{align*}
by introducing slack variables $\xi_i, \xi_i^{*}$ and solving the optimization problem
\begin{align*}
  &\frac{1}{2} \lVert \bm{w} \rVert^2 + C \sum_{i=1}^{n} (\xi_i + \xi_i^{*}) \rightarrow \min_{\bm{w}, b, \xi_i, \xi_i^{*}}, \\
  \text{s.t.}\quad &y_i - \hat{f}(\bm{x}_i) \leq \varepsilon + \xi_i, \\
  &\hat{f}(\bm{x}_i) - y_i \leq \varepsilon + \xi_i^{*}, \\
  &\xi_i, \xi_i^{*} \geq 0, \quad \forall i \in \{1, 2, \ldots, n\}.
\end{align*}
Introducing Lagrange multipliers $\alpha_i$ and $\alpha_i^{*}$, we obtain the dual problem with constraints
\begin{align*}
  \sum_{i=1}^{n} (\alpha_i - \alpha_i^{*}) = 0, \quad \alpha_i, \alpha_i^{*} \in [0, C].
\end{align*}
To extend this formulation to nonlinear regression, we introduce a mapping $\varphi: X \to \mathfrak{Z}$ from the input space $X$ to a feature space $\mathfrak{Z}$, together with the kernel function $k(\bm{x}_i, \bm{x}) = \langle \varphi(\bm{x}_i), \varphi(\bm{x}) \rangle$.
The prediction function can then be written as
\begin{align}
  \hat{f}(\bm{x}) = \sum_{i=1}^{n} (\alpha_i - \alpha_i^{*}) k(\bm{x}_i, \bm{x}) + b. \label{eq:svr-reg-kernel_b}
\end{align}
The KKT conditions imply $\alpha_i \alpha_i^{*} = 0$~\cite[Sec.~1.4]{Smola2004}.
In this paper, we consider the standard $\varepsilon$-insensitive SVR and assume $\varepsilon > 0$.
We focus on minimizing the prediction function~\cref{eq:svr-reg-kernel_b} with respect to the input variable $\bm{x}$.
Since $b$ does not depend on $\bm{x}$, we omit $b$ in what follows for notational simplicity and redefine
\begin{align}
  \hat{f}(\bm{x}) = \sum_{i=1}^{n} (\alpha_i - \alpha_i^{*}) k(\bm{x}_i, \bm{x}). \label{eq:svr-reg-kernel}
\end{align}

%%%%%%%%%%%%%%%%%%%%%%%%%%%%%%%
\subsection{DC Decomposition of the SVR Prediction Function}\label{sec:dc-decomposition}
As shown in~\cref{eq:svr-reg-kernel}, the SVR prediction function is expressed as a linear combination of the kernel functions $k(\bm{x}, \bm{x}_i)$.
Treating $\bm{x}_i$ as a constant, we set
\begin{align*}
  \phi_i(\bm{x}) = k(\bm{x}, \bm{x}_i).
\end{align*}
Focusing on the differences of Lagrange multipliers, we further define
\begin{align}
  \Delta \alpha_i & = \alpha_i - \alpha_i^{*},                                \notag\\
  S_{+} & = \{\, i \mid \alpha_i - \alpha_i^* \ge 0 \,\}, \label{eq:S_plus}\\
  S_{-} & = \{\, i \mid \alpha_i - \alpha_i^* < 0 \,\}.  \label{eq:S_minus}
\end{align}
Samples with $\Delta\alpha_i = 0$ belong to $S_+$ by this definition, but the complementarity condition $\alpha_i \alpha_i^{*} = 0$~\cite[Sec.~1.4]{Smola2004} forces $\alpha_i = \alpha_i^* = 0$ in this case, so they make no contribution to $\hat{f}(\bm{x})$ and can effectively be ignored in what follows.
We can then rewrite~\cref{eq:svr-reg-kernel} as
\begin{align}
  \hat{f}(\bm{x}) = \sum_{i\in S_{+}}\Delta \alpha_i \phi_i(\bm{x}) + \sum_{i\in S_{-}}\Delta \alpha_i \phi_i(\bm{x}). \label{eq:svr-reg2-1}
\end{align}
The complementarity condition $\alpha_i \alpha_i^{*} = 0$ further yields $\alpha_i^* = 0$ for $i \in S_+$ and $\alpha_i = 0$ for $i \in S_-$.
Therefore
\begin{align*}
  \Delta\alpha_i =
  \begin{cases}
    \alpha_i       & (i \in S_+), \\
    -\alpha_i^{*}  & (i \in S_-),
  \end{cases}
\end{align*}
and substituting this into~\cref{eq:svr-reg2-1} gives
\begin{align}
  \hat{f}(\bm{x}) = \sum_{i\in S_{+}}\alpha_i \phi_i(\bm{x}) - \sum_{i\in S_{-}}\alpha^*_i \phi_i(\bm{x}). \label{eq:svr-reg2-2}
\end{align}

We now consider a DC decomposition of~\cref{eq:svr-reg2-2}.
The class of DC functions is closed under nonnegative linear combinations~\cite{PhamDinh1997}, 
and we exploit this property to construct a DC decomposition.
First, we define the DC decomposition $\phi_i = g_i - h_i$ of each kernel $\phi_i$ as follows~\cite{PhamDinh1998}:
\begin{align}
  g_i(\bm{x}) & = \frac{\rho}{2} \lVert \bm{x}\rVert^2, \label{eq:g_rbf} \\
  h_i(\bm{x}) & = \frac{\rho}{2} \lVert \bm{x}\rVert^2 - \phi_i(\bm{x}). \label{eq:h_rbf}
\end{align}
Here, $\rho$ is a positive constant; $g_i$ is clearly convex, and $h_i$ is convex for an appropriate choice of $\rho$.
Indeed, by~\cref{lem:kernel_eig}, every $h_i$ is convex when $\rho \geq 4\gamma\exp(-\tfrac{3}{2})$.
Using this decomposition, we can rewrite~\cref{eq:svr-reg2-2} as
\begin{align}
  \hat{f}(\bm{x}) 
   &= \sum_{i\in S_{+}}\alpha_i (g_i(\bm{x}) - h_i(\bm{x})) - \sum_{i\in S_{-}}\alpha^*_i (g_i(\bm{x}) - h_i(\bm{x})) \notag \\
   &= G(\bm{x}) - H(\bm{x}), \label{eq:svr-dc_bias}
\end{align}
where
\begin{align}
  G(\bm{x}) &= \sum_{i\in S_{+}}\alpha_i g_i(\bm{x}) + \sum_{i\in S_{-}}\alpha_i^* h_i(\bm{x}), \label{eq:G} \\
  H(\bm{x}) &= \sum_{i\in S_{+}}\alpha_i h_i(\bm{x}) + \sum_{i\in S_{-}}\alpha_i^* g_i(\bm{x}). \label{eq:H}
\end{align}
Since nonnegative linear combinations of convex functions are convex and $\alpha_i, \alpha_i^* \geq 0$, 
both $G$ and $H$ are convex when $\rho \geq 4\gamma\exp(-\tfrac{3}{2})$.
This yields the DC decomposition $\hat{f} = G - H$.

Substituting~\cref{eq:g_rbf,eq:h_rbf} into~\cref{eq:G,eq:H} yields
\begin{align}
  G(\bm{x}) &= \frac{C_{\alpha}\rho}{2} \lVert \bm{x}\rVert^2
                    - \sum_{i\in S_-} \alpha^*_i \phi_i(\bm{x}), \label{eq:G_explicit} \\
  H(\bm{x}) &= \frac{C_{\alpha}\rho}{2} \lVert \bm{x}\rVert^2
                    - \sum_{i\in S_+} \alpha_i \phi_i(\bm{x}),  \label{eq:H_explicit}
\end{align}
where
\begin{align*}
  C_{\alpha} = \sum_{i\in S_+} \alpha_i + \sum_{i\in S_-} \alpha^*_i.
\end{align*}
Now, splitting the equality constraint of the dual problem
\begin{align*}
  \sum_{i=1}^{n} (\alpha_i -\alpha_i^{*}) = 0,
\end{align*}
according to the definitions of $S_+$ and $S_-$ in~\cref{eq:S_plus,eq:S_minus}, we obtain
\begin{align*}
  \sum_{i\in S_+} (\alpha_i -\alpha_i^{*}) + \sum_{i\in S_-} (\alpha_i -\alpha_i^{*}) = 0.
\end{align*}
Since $\alpha_i, \alpha_i^* \geq 0$ and $\alpha_i \alpha_i^* = 0$, we have $\alpha_i^* = 0$ for $i \in S_+$ and $\alpha_i = 0$ for $i \in S_-$.
Therefore,
\begin{align*}
  \sum_{i\in S_+} \alpha_i - \sum_{i\in S_-} \alpha_i^{*} = 0,
\end{align*}
that is,
\begin{align*}
  \sum_{i\in S_+} \alpha_i = \sum_{i\in S_-} \alpha_i^{*}.
\end{align*}
Combined with $\alpha_i, \alpha_i^* \geq 0$, this yields
\begin{align}
  C_{\alpha} 
    = 2\sum_{i\in S_+}\alpha_i 
    = 2\sum_{i\in S_-}\alpha_i^{*} 
    \geq 0. \label{eq:Calpha_half}
\end{align}
This relation will be used repeatedly from~\cref{sec:convexity} onward.

We have thus established that $\hat{f} = G - H$ is a DC decomposition when $\rho \geq 4\gamma\exp(-\tfrac{3}{2})$.
However, the explicit forms of $G$ and $H$ in~\cref{eq:G_explicit,eq:H_explicit} suggest that the condition $\rho \geq 4\gamma\exp(-\tfrac{3}{2})$ ensuring convexity of each $h_i$ is stronger than necessary, and smaller $\rho$ may still ensure the convexity of $G$ and $H$.
This is because $G$ and $H$ are obtained as linear combinations weighted by the dual coefficients $\alpha_i, \alpha_i^*$, so that the combination can remain convex even when individual $h_i$ are not.
This observation is quantified in~\cref{sec:convexity,sec:convergence}, where we show that the lower bound for $\rho$ can be relaxed to $2\gamma\exp(-\tfrac{3}{2})$.

%%%%%%%%%%%%%%%%%%%%%%%%%%%%%%%
\subsection{Application of DCA and Derivation of the Subproblem}\label{sec:dca-subproblem}
\begin{algorithm}
\caption{DCA for the RBF-SVR Prediction Function}
\label{alg:dca}
\begin{algorithmic}[1]
\REQUIRE Initial point $\bm{x}^{(0)} \in \Omega$; convergence thresholds $\varepsilon_x > 0$, $\varepsilon_f > 0$; DC decomposition parameter $\rho > \rho_{\min}$~(\cref{thm:mu})
\STATE $t \leftarrow 0$
\REPEAT
  \STATE $\bm{s}^{(t)} \leftarrow \nabla H(\bm{x})\big|_{\bm{x}=\bm{x}^{(t)}}$
  \STATE $\bm{x}^{(t+1)} \leftarrow \argmin_{\bm{x} \in \Omega}\, F^{(t)}(\bm{x})$
  \STATE $t \leftarrow t + 1$
\UNTIL{$\lVert \bm{x}^{(t)} - \bm{x}^{(t-1)} \rVert \leq \varepsilon_x$ and $\lvert \hat{f}(\bm{x}^{(t)}) - \hat{f}(\bm{x}^{(t-1)}) \rvert \leq \varepsilon_f$}
\ENSURE $\bm{x}^{(t)}$
\end{algorithmic}
\end{algorithm}

In what follows, we assume $C_\alpha > 0$.
If $C_\alpha = 0$, all dual coefficients vanish, $\hat{f}$ in~\cref{eq:svr-reg-kernel} becomes identically zero, and the optimization problem is trivial.
Also, let $\Omega \subset \mathbb{R}^d$ be a non-empty, bounded, closed, convex set, which we take as the feasible region of the optimization.
All subsequent arguments hold for any $\Omega$ satisfying these conditions.

DCA is an iterative method for minimizing DC-represented functions, generating the next iterate $\bm{x}^{(t+1)}$ from the current iterate $\bm{x}^{(t)}$ until a convergence criterion is satisfied.
The key idea of DCA is to leverage, at each iterate $\bm{x}^{(t)}$, the first-order lower bound on $H(\bm{x})$ given by a subgradient $\bm{s}^{(t)} \in \partial H(\bm{x}^{(t)})$:
\begin{align}
  H(\bm{x}) \geq H(\bm{x}^{(t)}) + \langle \bm{s}^{(t)}, \bm{x} - \bm{x}^{(t)} \rangle. \label{eq:H_lower_bound}
\end{align}
\Cref{eq:H_lower_bound} holds for any $\bm{x}$ by the definition of the subgradient and the convexity of $H$.
Substituting~\cref{eq:H_lower_bound} into~\cref{eq:svr-dc_bias} yields the upper bound
\begin{align}
  \hat{f}(\bm{x}) &= G(\bm{x}) - H(\bm{x}) \notag \\
            &\leq G(\bm{x})- \{H(\bm{x}^{(t)}) + \langle \bm{s}^{(t)}, \bm{x} - \bm{x}^{(t)} \rangle\} \notag \\
            &= G(\bm{x}) - \langle \bm{s}^{(t)}, \bm{x} \rangle - H(\bm{x}^{(t)}) + \langle \bm{s}^{(t)}, \bm{x}^{(t)} \rangle. \label{eq:f_upper_bound}
\end{align}
Now, since the third and fourth terms in~\cref{eq:f_upper_bound} are constants, the upper bound on $\hat{f}(\bm{x})$ can be obtained by minimizing
\begin{align*}
  F^{(t)}(\bm{x}) &= G(\bm{x}) - \langle \bm{s}^{(t)}, \bm{x} \rangle.
\end{align*}
In summary, at the $t$-th iteration, DCA computes the next iterate $\bm{x}^{(t+1)}$ by solving the optimization problem
\begin{align}
  \bm{x}^{(t+1)} &\in \argmin_{\bm{x}\in \Omega} F^{(t)}(\bm{x}). \label{eq:dca_subprob}
\end{align}
The DCA procedure is summarized in~\cref{alg:dca}.

DCA is an iterative method that updates the solution by solving a subproblem formulated as a convex optimization problem at each iteration.
When the subproblem is convex, its global optimum can be computed numerically using standard convex optimization methods.
Accordingly, this paper considers only simple constraints such as bound constraints, leaving complex constraints aside.
The analysis focuses primarily on the structure and behavior of the subproblem arising from the objective function.

%%%%%%%%%%%%%%%%%%%%%%%%%%%%%%%%%%%%%%%%%%
\section{Analysis of the DCA Subproblem for RBF-SVR}\label{sec:dca-analysis}
%%%%%%%%%%%%%%%%%%%%%%%%%%%%%%%%%%%%%%%%%%
In this section, we apply the DCA introduced in~\cref{sec:dca-subproblem} to the RBF-SVR prediction function.
Specifically, we first derive the subproblem explicitly in~\cref{sec:rbf-svr-subproblem} and discuss its geometric structure in~\cref{sec:structure-analysis}.
Next, in~\cref{sec:convexity}, we analytically derive the lower bound on $\rho$ that ensures the subproblem is convex.
In~\cref{sec:convergence}, we derive in closed form the strong convexity parameter $\mu$ of the DC components $G$ and $H$ and the upper bound $L$ on the gradient Lipschitz constant.
Through these analyses, we show that the central quantity $C_\alpha\rho$, which characterizes the convergence properties of DCA, is determined in closed form by the SVR hyperparameters and the learned dual coefficients.

\subsection{Subproblem for the RBF-SVR Prediction Function}\label{sec:rbf-svr-subproblem}
In this paper, we apply DCA to the nonconvex optimization problem of minimizing the RBF-SVR prediction function.
Specifically, we set
\begin{align*}
  \phi_i(\bm{x}) = \exp(-\gamma\lVert \bm{x}-\bm{x}_i \rVert^2)
\end{align*}
and explicitly derive the optimization problem~\cref{eq:dca_subprob} that DCA solves at each step.
Since the convex function $H(\bm{x})$ given in~\cref{eq:H_explicit} is of class $C^{\infty}$, the subgradient coincides with the gradient.
Therefore, by~\cref{eq:H_explicit}, we have
\begin{align*}
  \nabla H(\bm{x}) = C_{\alpha}\rho \bm{x}
    + 2\gamma \sum_{i\in S_+} \alpha_i \phi_i(\bm{x}) (\bm{x}-\bm{x}_i).
\end{align*}
Evaluating this at the iterate $\bm{x}^{(t)}$ yields
\begin{align}
  \bm{s}^{(t)}
    &= \nabla H(\bm{x})\big|_{\bm{x}=\bm{x}^{(t)}} \notag \\
    &= C_{\alpha}\rho \bm{x}^{(t)}
       + 2\gamma \sum_{i\in S_+} \alpha_i \phi_i(\bm{x}^{(t)}) (\bm{x}^{(t)}-\bm{x}_i).
    \label{eq:grad_H}
\end{align}
Now, using~\cref{eq:G_explicit}, the objective function~$F^{(t)}(\bm{x})$ of the subproblem~\cref{eq:dca_subprob} can be written as
\begin{align}
  F^{(t)}(\bm{x}) = \frac{C_{\alpha} \rho}{2} \lVert \bm{x}\rVert^2 
                    - \langle \bm{s}^{(t)}, \bm{x} \rangle
                    - \sum_{i\in S_-} \alpha^*_i \phi_i(\bm{x}). \label{eq:F_t_explicit}
\end{align}
The first two terms of~\cref{eq:F_t_explicit} together form a quadratic in $\bm{x}$.
Defining
\begin{align}
  \bm{v}^{(t)} &= \frac{1}{C_{\alpha}\rho} \bm{s}^{(t)} \notag \\
               &= \bm{x}^{(t)} + \frac{2\gamma}{C_{\alpha}\rho} \sum_{i\in S_+} \alpha_i \phi_i(\bm{x}^{(t)}) (\bm{x}^{(t)}-\bm{x}_i), \label{eq:v_def}
\end{align}
and completing the square, we can rewrite~$F^{(t)}(\bm{x})$ as
\begin{align}
  F^{(t)}(\bm{x}) &= \frac{C_{\alpha} \rho}{2} \lVert \bm{x}\rVert^2 
                     - \langle \bm{s}^{(t)}, \bm{x} \rangle 
                     - \sum_{i\in S_-} \alpha^*_i \phi_i(\bm{x}) \notag \\
                  &= \frac{C_{\alpha} \rho}{2} \lVert \bm{x} - \bm{v}^{(t)} \rVert^2 
                     - \frac{1}{2C_{\alpha}\rho} \lVert \bm{s}^{(t)} \rVert^2
                     - \sum_{i\in S_-} \alpha^*_i \phi_i(\bm{x}).   \label{eq:quadratic_completion}
\end{align}
Here, $\bm{v}^{(t)}$ arises naturally as the center of the quadratic in $F^{(t)}(\bm{x})$.
The second term in the second line of~\cref{eq:quadratic_completion} is a constant independent of $\bm{x}$ and does not affect the $\argmin$.
Therefore, the subproblem~\cref{eq:dca_subprob} can be written in its final form as
\begin{align}
  \bm{x}^{(t+1)} &\in \argmin_{\bm{x}\in\Omega} \left\{
                        F^{(t)}(\bm{x})
                      \right\} \notag \\
                 &=  \argmin_{\bm{x}\in\Omega} \left\{
                        \frac{C_{\alpha} \rho}{2} \lVert \bm{x} - \bm{v}^{(t)} \rVert^2 
                        - \sum_{i\in S_-} \alpha^*_i \phi_i(\bm{x})
                      \right\}. \label{eq:dca-subprob3}
\end{align}

%%%%%%%%%%%%%%%%%%%%%%%%%%%%%%
\subsection{Geometric Structure of the Subproblem}\label{sec:structure-analysis}
The subproblem~\cref{eq:dca-subprob3} can be viewed as an optimization problem whose objective consists of a quadratic proximal term centered at $\bm{v}^{(t)}$ and a kernel sum based on $\{\bm{x}_i\}_{i\in S_-}$.
Moreover, $\bm{v}^{(t)}$ depends on $\{\bm{x}_i\}_{i\in S_+}$.
Then, let $W^{(t)} = 2\gamma \sum_{i\in S_+} \alpha_i \phi_i(\bm{x}^{(t)})$, which is positive under the assumption $C_\alpha > 0$, and define the centroid of samples in $S_+$ as
\begin{align*}
  \bm{x}^{(t)}_{+} = \frac{2\gamma}{W^{(t)}} \sum_{i\in S_+} \alpha_i \phi_i(\bm{x}^{(t)}) \bm{x}_i.
\end{align*}
With these notations, \cref{eq:v_def} can be rewritten as
\begin{align*}
  \bm{v}^{(t)} = \bm{x}^{(t)} + \frac{W^{(t)}}{C_{\alpha}\rho} (\bm{x}^{(t)} - \bm{x}^{(t)}_{+}).
\end{align*}
From this, we see that $\bm{v}^{(t)}$ lies on the extension of the line from the centroid $\bm{x}^{(t)}_+$ of $S_+$ through $\bm{x}^{(t)}$.

Moreover, if we focus on only the first term of~\cref{eq:dca-subprob3}, the optimal solution of the subproblem can be regarded as $\bm{v}^{(t)}$, the point at which this term attains its minimum.
On the other hand, the second term of~\cref{eq:dca-subprob3} is the kernel sum $-\sum_{i\in S_-} \alpha^*_i \phi_i(\bm{x})$ based on samples in $S_-$.
Since this term drives $\bm{x}$ in the direction in which $\phi_i(\bm{x})$ increases, it has the effect of moving $\bm{x}$ toward the samples in $S_-$.
Therefore, the optimal solution $\bm{x}^{(t+1)}$ of the subproblem~\cref{eq:dca-subprob3} is kept near $\bm{v}^{(t)}$ by the quadratic proximal term but is displaced from $\bm{v}^{(t)}$ by the influence of $S_-$.
In other words, we can view the samples in $S_+$ as determining the reference point of the search in the subproblem~\cref{eq:dca-subprob3}, while the samples in $S_-$ determine the direction of its displacement.
Although $\bm{v}^{(t)}$ depends on $\{\bm{x}_i\}_{i \in S_+}$, these samples are constants, so $\bm{v}^{(t)}$ does not affect the curvature of the subproblem~\cref{eq:dca-subprob3}.
Hence, the convexity analysis reduces to controlling the curvature arising from the second term via $\rho$.

%%%%%%%%%%%%%%%%%%%%%%%%%%%%%%
\subsection{Convexity of the Subproblem} \label{sec:convexity}
The convexity of the DCA subproblem~\cref{eq:dca-subprob3} depends on the DC decomposition parameter $\rho$.
In this section, we analytically derive the lower bound $\rho_{\min}$ on $\rho$ that ensures the subproblem is a convex optimization problem at any iterate $\bm{x}^{(t)}$.

\begin{proposition}[Convexity of the Subproblem] \label{prop:convexity}
For the DC decomposition~\cref{eq:svr-dc_bias} of the SVR prediction function~\cref{eq:svr-reg-kernel} with a Gaussian RBF kernel, the subproblem~\cref{eq:dca-subprob3} is a convex optimization problem provided that
\begin{align}
  \rho \geq \rho_{\min} = 2\gamma \exp\!\left(-\tfrac{3}{2}\right). \label{eq:rho_min}
\end{align}
\end{proposition}

\begin{proof}
A sufficient condition for the subproblem~\cref{eq:dca-subprob3} to be convex is that the Hessian of $F^{(t)}(\bm{x})$ is positive semidefinite for all $\bm{x}\in \Omega$:
\begin{align*}
  \nabla^2 F^{(t)}(\bm{x}) \succeq 0.
\end{align*}
Here, the Hessian of $F^{(t)}(\bm{x})$ is
\begin{align*}
  \nabla^2 F^{(t)}(\bm{x}) 
    &= C_{\alpha}\rho I 
       - \sum_{i\in S_-} \alpha^*_i \nabla^2 \phi_i(\bm{x}),
\end{align*}
so a sufficient condition for $\nabla^2 F^{(t)}(\bm{x}) \succeq 0$ is
\begin{align*}
  C_{\alpha}\rho I 
    \succeq \sum_{i\in S_-} \alpha^*_i \nabla^2 \phi_i(\bm{x}).
\end{align*}
By~\cref{lem:kernel_eig}, the maximum eigenvalue of the Hessian of the Gaussian RBF $\phi_i(\bm{x}) = \exp(-\gamma\lVert \bm{x}-\bm{x}_i \rVert^2)$ satisfies
\begin{align*}
  \lambda_{\max} (\nabla^2 \phi_i(\bm{x})) \leq 4\gamma \exp(-\tfrac{3}{2})
\end{align*}
for any $\bm{x}$.
Therefore, for any $\bm{x}$,
\begin{align*}
  \nabla^2 \phi_i(\bm{x}) \preceq 4\gamma \exp(-\tfrac{3}{2}) I.
\end{align*}
Hence,
\begin{align*}
  \sum_{i\in S_-} \alpha^*_i \nabla^2 \phi_i(\bm{x}) \preceq \left(4\gamma \exp(-\tfrac{3}{2}) \sum_{i\in S_-} \alpha^*_i \right) I.
\end{align*}
As a sufficient condition, we obtain
\begin{align*}
  C_{\alpha}\rho I \succeq \left(4\gamma \exp(-\tfrac{3}{2}) \sum_{i\in S_-} \alpha^*_i\right) I.
\end{align*}
That is, using~\cref{eq:Calpha_half}, we obtain
\begin{align*}
  \rho &\geq \frac{4\gamma \exp(-\tfrac{3}{2})}{C_{\alpha}} \sum_{i\in S_-} \alpha^*_i \notag \\
       &= \frac{4\gamma \exp(-\tfrac{3}{2})}{2 \sum_{i\in S_-} \alpha_i^*} \sum_{i\in S_-} \alpha^*_i \notag \\
       &= 2\gamma \exp(-\tfrac{3}{2}).
\end{align*}
Therefore, when $\rho \geq \rho_{\min}$, $\nabla^2 F^{(t)}(\bm{x}) \succeq 0$ holds for all $\bm{x}$, and the subproblem~\cref{eq:dca-subprob3} is a convex optimization problem.
\end{proof}

%%%%%%%%%%%%%%%%%%%%%%%%%%%%%%%%%%
\subsection{Strong Convexity, Smoothness Bound, and Convergence Guarantee}\label{sec:convergence}
In this section, for the DC decomposition $\hat{f} = G - H$ of the RBF-SVR prediction function, we analytically derive a lower bound $\mu$ on the strong convexity parameters of the DC components $G, H$ and an upper bound $L$ on their gradient Lipschitz constants.
Furthermore, using these, we show that the descent inequality holds at each iteration of DCA and that any accumulation point of the iterate sequence satisfies the criticality condition.
We thereby establish that the convergence properties of DCA to critical points can be evaluated prior to running the iterations in terms of quantities determined by the SVR training results and hyperparameters.
In what follows, we define the strong convexity parameter of any proper closed convex function $\Phi$ as
\begin{align*}
  \sigma(\Phi) = \sup\{\, \eta \geq 0 \mid \Phi - \tfrac{\eta}{2}\lVert \cdot \rVert^2 \text{ is convex}\,\}.
\end{align*}

\begin{theorem}[Strong Convexity of the DC Components] \label{thm:mu}
In the DC decomposition~\cref{eq:svr-dc_bias} of the SVR prediction function~\cref{eq:svr-reg-kernel} with a Gaussian RBF kernel, suppose $\rho > \rho_{\min}$.
Then the strong convexity parameters of $G$ and $H$ satisfy
\begin{align*}
  \sigma(G),\, \sigma(H) \geq \mu = C_{\alpha}\bigl\{\rho - 2\gamma\exp\!\left(-\tfrac{3}{2}\right)\bigr\}.
\end{align*}
\end{theorem}

\begin{proof}
By~\cref{eq:G_explicit}, the Hessian of $G$ is
\begin{align*}
  \nabla^2 G(\bm{x}) = C_{\alpha}\rho\, I - \sum_{i\in S_-} \alpha_i^* \nabla^2 \phi_i(\bm{x}).
\end{align*}
By~\cref{lem:kernel_eig}, $\lambda_{\max}(\nabla^2 \phi_i(\bm{x})) \leq 4\gamma\exp(-\tfrac{3}{2})$, and since $\nabla^2 \phi_i(\bm{x})$ is symmetric,
\begin{align*}
  \nabla^2 \phi_i(\bm{x}) \preceq 4\gamma\exp(-\tfrac{3}{2}) I
\end{align*}
holds for any $\bm{x}$.
Noting that $\alpha_i^* \geq 0$, summing with these nonnegative coefficients yields
\begin{align*}
  \sum_{i\in S_-} \alpha_i^* \nabla^2 \phi_i(\bm{x}) \preceq 4\gamma \exp\!\left(-\tfrac{3}{2}\right) \sum_{i\in S_-}\alpha_i^* I.
\end{align*}
Combining this with~\cref{eq:Calpha_half}, under the assumption $\rho > \rho_{\min}$,
\begin{align*}
  \nabla^2 G(\bm{x})
    &\succeq \left\{C_{\alpha}\rho - 4\gamma \exp\!\left(-\tfrac{3}{2}\right) \sum_{i\in S_-} \alpha^*_i\right\} I \notag \\
    &= C_{\alpha}\left\{\rho - 2\gamma \exp\!\left(-\tfrac{3}{2}\right)\right\} I \notag \\
    &=: \mu I
\end{align*}
holds for any $\bm{x}$.
That is, $\sigma(G) \geq \mu$.

Next, we evaluate the strong convexity parameter of $H$.
By~\cref{eq:H_explicit}, the Hessian of $H$ is
\begin{align*}
  \nabla^2 H(\bm{x}) = C_{\alpha}\rho\, I - \sum_{i\in S_+} \alpha_i \nabla^2 \phi_i(\bm{x}).
\end{align*}
By the same argument as for $G$, $\nabla^2 \phi_i(\bm{x}) \preceq 4\gamma\exp(-\tfrac{3}{2}) I$ holds for any $\bm{x}$.
Noting that $\alpha_i \geq 0$, summing with these nonnegative coefficients yields
\begin{align*}
  \sum_{i\in S_+} \alpha_i \nabla^2 \phi_i(\bm{x}) \preceq 4\gamma \exp\!\left(-\tfrac{3}{2}\right) \sum_{i\in S_+} \alpha_i I.
\end{align*}
Combining this with~\cref{eq:Calpha_half}, under the assumption $\rho > \rho_{\min}$,
\begin{align*}
  \nabla^2 H(\bm{x})
    &\succeq \left\{C_{\alpha}\rho - 4\gamma \exp\!\left(-\tfrac{3}{2}\right) \sum_{i\in S_+} \alpha_i\right\} I \notag \\
    &= C_{\alpha}\left\{\rho - 2\gamma \exp\!\left(-\tfrac{3}{2}\right)\right\} I \notag \\
    &= \mu I
\end{align*}
holds for any $\bm{x}$.
That is, $\sigma(H) \geq \mu$.
\end{proof}

\begin{theorem}[Gradient Lipschitz Bound for the DC Components] \label{thm:L}
Under the assumptions of~\cref{thm:mu}, for any $\bm{x}\in\mathbb{R}^d$, the spectral norms of the Hessians of $G$ and $H$ satisfy
\begin{align*}
  \lVert \nabla^2 G(\bm{x}) \rVert_2,\, \lVert \nabla^2 H(\bm{x}) \rVert_2 \leq L = C_{\alpha}(\rho + \gamma).
\end{align*}
\end{theorem}

\begin{proof}
Since $\sigma(G) \geq \mu > 0$, $G$ is strongly convex under $\rho > \rho_{\min}$, and its Hessian $\nabla^2 G(\bm{x})$ is positive definite for any $\bm{x}$.
Therefore,
\begin{align*}
  \lVert \nabla^2 G(\bm{x}) \rVert_2 = \lambda_{\max}(\nabla^2 G(\bm{x}))
\end{align*}
holds, and the evaluation of the gradient Lipschitz constant reduces to that of the maximum eigenvalue of $\nabla^2 G(\bm{x})$.
From~\cref{eq:G_explicit},
\begin{align*}
  \lambda_{\max}(\nabla^2 G(\bm{x})) = C_{\alpha}\rho - \lambda_{\min}\!\left(\sum_{i\in S_-} \alpha_i^* \nabla^2 \phi_i(\bm{x})\right)
\end{align*}
holds.
By~\cref{lem:kernel_eig}, $\lambda_{\min}(\nabla^2\phi_i(\bm{x})) \geq -2\gamma$, and since $\nabla^2\phi_i(\bm{x})$ is symmetric,
\begin{align*}
  \nabla^2 \phi_i(\bm{x}) \succeq -2\gamma I
\end{align*}
holds for any $\bm{x}$.
Noting that $\alpha_i^* \geq 0$, summing with these nonnegative coefficients yields
\begin{align*}
  \sum_{i\in S_-} \alpha_i^* \nabla^2 \phi_i(\bm{x}) \succeq -2\gamma \sum_{i\in S_-} \alpha_i^* I.
\end{align*}
Comparing the minimum eigenvalues of both sides, we obtain
\begin{align*}
  \lambda_{\min}\!\left(\sum_{i\in S_-} \alpha_i^* \nabla^2 \phi_i(\bm{x})\right)
  \geq -2\gamma \sum_{i\in S_-} \alpha_i^*.
\end{align*}
Using $\sum_{i\in S_-} \alpha_i^* = C_{\alpha}/2$ from~\cref{eq:Calpha_half},
\begin{align*}
  \lambda_{\max}(\nabla^2 G(\bm{x}))
  &\leq C_{\alpha}\rho + 2\gamma \sum_{i\in S_-} \alpha_i^* \\
  &= C_{\alpha}\rho + C_{\alpha}\gamma \\
  &= C_{\alpha}(\rho + \gamma)
\end{align*}
holds for any $\bm{x}$.

Similarly for $H$, since $\sigma(H) \geq \mu > 0$, $H$ is strongly convex, and by the same argument as for $G$ using~\cref{eq:Calpha_half}, we obtain
\begin{align*}
  \lVert \nabla^2 H(\bm{x}) \rVert_2 &= \lambda_{\max}(\nabla^2 H(\bm{x})) \\
                                     &\leq C_{\alpha}(\rho + \gamma).
\end{align*}

Therefore, the gradient Lipschitz constants of both $G$ and $H$ are bounded above by the same constant
\begin{align*}
  L = C_{\alpha}(\rho + \gamma).
\end{align*}
\end{proof}

The strong convexity and smoothness bounds $\mu$ and $L$ from~\cref{thm:mu,thm:L} are both analytically determined prior to running the DCA iterations from only three quantities: $C_{\alpha}$, which is fixed by the SVR training results; the RBF kernel parameter $\gamma$; and the DC decomposition parameter $\rho$.
Moreover, $\mu$ and $L$ share the common leading term $C_{\alpha}\rho$, but their perturbation terms are asymmetric.
Specifically, we have
\begin{align*}
  \mu &= C_{\alpha}\rho - 2C_{\alpha}\gamma \exp\!\left(-\tfrac{3}{2}\right),\\
  L &= C_{\alpha}\rho + C_{\alpha}\gamma,
\end{align*}
and the coefficient of the perturbation term is $2\gamma\exp(-\tfrac{3}{2})$ for $\mu$ and $\gamma$ for $L$.
This asymmetry stems from~\cref{lem:kernel_eig}, where the upper bound $4\gamma\exp(-\tfrac{3}{2})$ on the maximum eigenvalue and the lower bound $-2\gamma$ on the minimum eigenvalue of each individual kernel function's Hessian are two distinct values derived independently of each other.

These two constants play different roles in the convergence analysis of DCA.
The strong convexity lower bound $\mu$ guarantees the per-iteration descent of the objective function value, as will be shown by the descent inequality below.
On the other hand, the gradient Lipschitz upper bound $L$ for $G$ characterizes the smoothness of the DCA subproblem and serves as the reference smoothness constant for the subproblem in applications to DC-FW~\cite{Maskan2025,Pokutta2025}.
Thus, the strong convexity bound $\mu$ characterizes the descent of DCA itself, while the smoothness bound $L$ characterizes the structure of the subproblem.
Since both constants are dominated by their common leading term $C_{\alpha}\rho$, this quantity emerges as the central one summarizing the convergence properties of DCA.

In what follows, we use the result of~\cref{thm:mu} to derive the descent inequality at each iteration of DCA.
To address the convergence of constrained DCA in its general form, we introduce the indicator function of $\Omega$,
\begin{align*}
  \chi_{\Omega}(\bm{x}) =
  \begin{cases}
    0       & (\bm{x} \in \Omega), \\
    +\infty & (\bm{x} \notin \Omega),
  \end{cases}
\end{align*}
and incorporate it into $G$ to define
\begin{align*}
  \tilde{G}(\bm{x}) = G(\bm{x}) + \chi_{\Omega}(\bm{x}).
\end{align*}
The indicator function $\chi_{\Omega}$ on the nonempty closed convex set $\Omega$ is a proper closed convex function.
By the construction in this paper, $G$ is also proper closed convex and of class $C^{\infty}$, and if $G$ is strongly convex on $\Omega$, the strong convexity parameter of $\tilde{G}$ satisfies
\begin{align*}
  \sigma(\tilde{G}) \geq \sigma(G).
\end{align*}
Then, each iteration of~\cref{alg:dca} can be equivalently rewritten as
\begin{align}
  \bm{x}^{(t+1)} = \argmin_{\bm{x}}\bigl\{\tilde{G}(\bm{x}) - \langle \bm{s}^{(t)}, \bm{x}\rangle\bigr\}, \label{eq:dca_subprob_tilde}
\end{align}
which corresponds to the standard formulation for treating constrained DCA within the framework of unconstrained DCA.
Since $H$ is also proper closed convex and of class $C^{\infty}$ by the construction in this paper, the subdifferential $\partial H$ coincides with the gradient $\nabla H$, as noted in~\cref{eq:grad_H}.

According to the general convergence theory of DCA~\cite{PhamDinh1997}, when $\tilde{G}$ and $H$ in the DC decomposition $\hat{f} + \chi_{\Omega} = \tilde{G} - H$ are each proper closed convex functions satisfying $\sigma(\tilde{G}) > 0$ and $\sigma(H) > 0$, the descent of the objective function value at each iteration of DCA is guaranteed by
\begin{align}
  \hat{f}(\bm{x}^{(t+1)}) \leq \hat{f}(\bm{x}^{(t)}) - \frac{\sigma(\tilde{G})+\sigma(H)}{2}\lVert \bm{x}^{(t+1)} - \bm{x}^{(t)} \rVert^2  \label{eq:dca_descent_guarantee}
\end{align}
~\cite[Prop.~2(i)]{PhamDinh1997}.
For the RBF-SVR considered in this paper, this general theorem is applicable, and the descent inequality can be written down explicitly in terms of the strong convexity lower bound $\mu$.

\begin{corollary}[Descent Inequality of DCA] \label{cor:descent}
Under the assumptions of~\cref{thm:mu}, the sequence of iterates $\{\bm{x}^{(t)}\}$ generated by~\cref{alg:dca} satisfies
\begin{align}
  \hat{f}(\bm{x}^{(t+1)}) \leq \hat{f}(\bm{x}^{(t)}) - \mu \lVert \bm{x}^{(t+1)} - \bm{x}^{(t)} \rVert^2 \label{eq:descent}
\end{align}
at each iteration $t$.
\end{corollary}

\begin{proof}
Since both $G$ and $H$ are proper closed convex functions of class $C^{\infty}$, and $\tilde{G} = G + \chi_{\Omega}$ is also a proper closed convex function, the assumptions of~\cite[Prop.~2(i)]{PhamDinh1997} are satisfied.
By~\cref{thm:mu}, $\sigma(G), \sigma(H) \geq \mu$.
Moreover, combining $\sigma(\tilde{G}) \geq \sigma(G)$ shown above with~\cref{thm:mu} yields $\sigma(\tilde{G}) \geq \sigma(G) \geq \mu$.
Substituting these into~\cref{eq:dca_descent_guarantee} yields~\cref{eq:descent}.
\end{proof}

Under the assumption $\rho > \rho_{\min}$, we have $\mu = C_{\alpha}\bigl\{\rho - 2\gamma\exp(-\tfrac{3}{2})\bigr\} > 0$, so~\cref{cor:descent} guarantees that the descent of the objective function value at each iteration of DCA is at least $\mu \lVert \bm{x}^{(t+1)} - \bm{x}^{(t)} \rVert^2$.
Note that~\cref{cor:descent} is a result for the ideal case where the subproblem~\cref{eq:dca_subprob_tilde} is solved exactly at each iteration; the behavior when the subproblem is solved approximately by numerical computation is empirically verified in~\cref{sec:experiments}.

In what follows, we use the descent inequality of~\cref{cor:descent} to show that any accumulation point of the iterate sequence $\{\bm{x}^{(t)}\}$ generated by DCA satisfies the criticality condition.
\begin{corollary}[Criticality of Every Accumulation Point of the Iterate Sequence] \label{cor:critical}
Suppose $\rho > \rho_{\min}$.
Every accumulation point $\bm{x}^*$ of the iterate sequence $\{\bm{x}^{(t)}\}$ generated by running~\cref{alg:dca} indefinitely without the stopping condition satisfies the criticality condition
\begin{align}
  0 \in \partial G(\bm{x}^*) - \nabla H(\bm{x}^*) + N_{\Omega}(\bm{x}^*). \label{eq:critical_point}
\end{align}
Here, $N_{\Omega}(\bm{x}^*)$ denotes the normal cone of $\Omega$ at $\bm{x}^*$.
\end{corollary}

\begin{proof}
From the descent inequality~\cref{eq:descent} of~\cref{cor:descent}, as long as $\bm{x}^{(t+1)} \neq \bm{x}^{(t)}$, we have $\hat{f}(\bm{x}^{(t+1)}) < \hat{f}(\bm{x}^{(t)})$, and thus the objective function value decreases monotonically.
Here, since the Gaussian RBF kernel satisfies $0 < \phi_i(\bm{x}) \leq 1$, from~\cref{eq:svr-reg2-2} and the nonnegativity $\alpha_i, \alpha_i^* \geq 0$,
\begin{align}
  \hat{f}(\bm{x})
  &\geq -\sum_{i\in S_-}\alpha_i^*\phi_i(\bm{x}) \notag \\
  &\geq -\sum_{i\in S_-}\alpha_i^* \notag \\
  &= -\frac{C_{\alpha}}{2} \label{eq:lower_bound}
\end{align}
holds for any $\bm{x} \in \mathbb{R}^d$, and thus $\hat{f}$ is bounded below.
Moreover, summing the descent inequality~\cref{eq:descent} over $t = 0, 1, \dots, T-1$ yields
\begin{align*}
  \mu \sum_{t=0}^{T-1} \lVert \bm{x}^{(t+1)} - \bm{x}^{(t)} \rVert^2
  &\leq \hat{f}(\bm{x}^{(0)}) - \inf_{\bm{x}\in\Omega} \hat{f}(\bm{x}) \\
  &< +\infty.
\end{align*}
Therefore, the series $\sum_{t=0}^{\infty}\|\bm{x}^{(t+1)} - \bm{x}^{(t)}\|^2$ converges, and hence $\lVert \bm{x}^{(t+1)} - \bm{x}^{(t)} \rVert \to 0$ as $t \to \infty$.

According to~\cite[Thm.~3(iv)]{PhamDinh1997}, the prerequisites for every accumulation point of the sequence $\{\bm{x}^{(t)}\}$ to be a critical point are:
\begin{itemize}
  \item $\{\bm{x}^{(t)}\}$ and $\{\bm{s}^{(t)}\}$ are bounded;
  \item $\inf_{\bm{x}\in\Omega} \hat{f}(\bm{x})$ is finite.
\end{itemize}
In the setting of this paper, we verify that these conditions are satisfied.
First, since $\Omega$ is compact, the sequence $\{\bm{x}^{(t)}\} \subset \Omega$ is bounded.
Furthermore, since $H$ is of class $C^{\infty}$ and $\bm{s}^{(t)} = \nabla H(\bm{x}^{(t)})$, by the continuity of $\nabla H$ on the compact set $\Omega$, the sequence $\{\bm{s}^{(t)}\}$ is also bounded.
Moreover, as shown in~\cref{eq:lower_bound}, $\hat{f}$ is bounded below on $\Omega$, and $\inf_{\bm{x}\in\Omega}\hat{f}(\bm{x})$ is finite.
Therefore, the prerequisites of~\cite[Thm.~3(iv)]{PhamDinh1997} are satisfied.

Under these prerequisites, for any accumulation point $\bm{x}^*$, there exists $\bm{s}^* \in \partial H(\bm{x}^*) \cap \partial \tilde{G}(\bm{x}^*)$~\cite[Thm.~3(iv)]{PhamDinh1997}.
Since $H$ is a $C^{\infty}$ convex function, $\partial H(\bm{x}^*) = \{\nabla H(\bm{x}^*)\}$, and thus $\bm{s}^* = \nabla H(\bm{x}^*)$.
Therefore,
\begin{align*}
  \nabla H(\bm{x}^*) \in \partial \tilde{G}(\bm{x}^*)
\end{align*}
holds.
Here, since $G$ is a finite-valued convex function on $\mathbb{R}^d$, the subdifferential of $\tilde{G} = G + \chi_{\Omega}$ decomposes as
\begin{align*}
  \partial \tilde{G}(\bm{x}^*) = \partial G(\bm{x}^*) + N_{\Omega}(\bm{x}^*).
\end{align*}
Here, $N_{\Omega}(\bm{x}^*)$ is the normal cone of $\Omega$ at $\bm{x}^*$, defined by
\begin{align*}
  N_{\Omega}(\bm{x}^*) = \{\bm{\nu} \in \mathbb{R}^d \mid \langle \bm{\nu}, \bm{x} - \bm{x}^*\rangle \leq 0,\ \forall \bm{x} \in \Omega\},
\end{align*}
and since $\Omega$ is convex, $N_{\Omega}(\bm{x}^*) = \partial \chi_{\Omega}(\bm{x}^*)$.
Therefore, $\bm{x}^*$ satisfies~\cref{eq:critical_point}.
Since $\bm{x}^*$ was an arbitrary accumulation point of $\{\bm{x}^{(t)}\}$, every accumulation point is a critical point.
\end{proof}

In particular, if $\bm{x}^*$ is an interior point of $\Omega$, then $N_{\Omega}(\bm{x}^*) = \{\bm{0}\}$, and~\cref{eq:critical_point} reduces to the unconstrained critical condition $\nabla H(\bm{x}^*) \in \partial G(\bm{x}^*)$.
Note that DCA generally converges to critical points, and convergence to a global optimum is not guaranteed.

\section{Numerical Verification}\label{sec:experiments}
In this section, we numerically examine the extent to which $C_\alpha\rho$, identified in~\cref{sec:convergence} as the common leading term of the strong convexity and smoothness bounds $\mu$ and $L$, accounts for the convergence behavior of DCA.
Specifically, we apply DCA to the prediction functions of RBF-SVR trained on noisy observations from six benchmark functions, and verify the following three points:
\begin{itemize}
  \item $C_\alpha\rho$ largely accounts for both the iteration count of DCA and its dependence on the initial point;
  \item $C_\alpha\rho$ decomposes into two independent pathways, $C \to C_\alpha$ and $\gamma \to \rho$, and its primary variation is governed by the SVR hyperparameters $(C, \gamma)$;
  \item the normalized residual exhibits linearly convergent behavior.
\end{itemize}

%%%%%%%%%%%%%%%%%%%%%%%%%%%%%%%%%%
\subsection{Settings}\label{sec:settings}
%%%%%%%%%%%%%%%%%%%%%%%%%%%%%%%%%%
\subsubsection{Dataset}\label{sec:dataset}
In this paper, we generate observations by adding noise to the benchmark functions in order to build the SVR models.
In the analysis of DCA's behavior, the noise magnitude is an important factor that influences the number of support vectors and the smoothness of the model.
Therefore, assuming Gaussian noise, we generate the observations as
\begin{align*}
    y_i = f(\bm{x}_i) + \eta_i, \quad \eta_i \sim \mathcal{N}(0, \sigma_{\eta}^2).
\end{align*}
Here, using the standard deviation $s_f$ of the noise-free true values $\{f(\bm{x}_i)\}_{i=1}^n$, we define $\sigma_{\eta}$ as
\begin{align*}
    \sigma_{\eta} = \tau s_f, \quad \tau \in \mathcal{T},
\end{align*}
where $\mathcal{T} \subset \mathbb{R}_{\ge 0}$ is a prescribed finite set.
By constructing models for multiple $\tau \in \mathcal{T}$ and comparing how DCA behaves, we evaluate the effect of the noise magnitude on the convergence behavior of DCA.
In the experiments, we set the noise levels to $\mathcal{T} = \{0.0,\, 0.1,\, 0.3,\, 0.6\}$.
For each benchmark function and each noise level $\tau$, a single noisy dataset was generated and shared across all hyperparameter settings $(C, \varepsilon, \gamma)$ with the random seed fixed to $0$, so that differences across models for the same benchmark function and noise level reflect the hyperparameters rather than the noise realization.

The number of samples substantially affects both the coefficients of the learned prediction function and the number of support vectors.
On the other hand, including the sample size as another varying factor would entangle its effect with the noise level and the hyperparameters, making it difficult to isolate the contribution of each factor to DCA's behavior.
Therefore, we fix the sample size at $n = 300$.
The input points $\{\bm{x}_i\}_{i=1}^{n}$ were drawn independently from the uniform distribution on $[0,1]^2$ with the random seed fixed to $0$, and reused across all benchmark functions and hyperparameter settings.

In this experiment, we use six benchmark functions: the Branin RCOS, Himmelblau, Six-hump Camel, Rastrigin, Ackley, and Levy functions~\cite{Jamil2013,Surjanovic2013}.
Each function is linearly transformed from its original domain to $[0,1]^2$.
\Cref{fig:benchmark_landscapes} shows the landscapes of the six benchmark functions after linear mapping to the unit square $[0,1]^2$.

\begin{figure}[htbp]
  \centering
  \includegraphics[width=\textwidth]{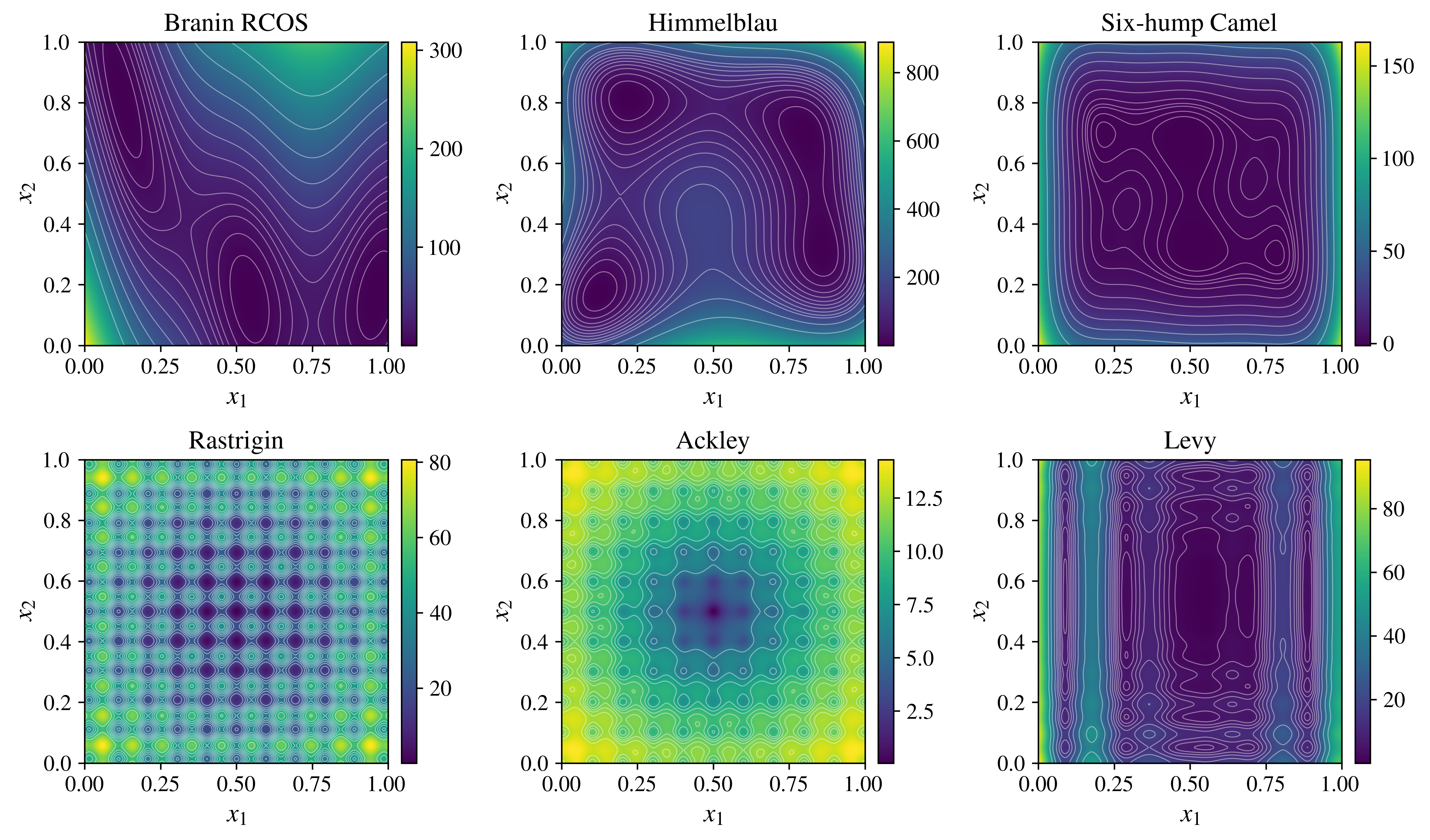}
  \caption{Heatmap visualizations of the six benchmark functions after linear mapping to the unit square $[0,1]^2$. Contour lines indicate level sets. The functions exhibit diverse multimodal structures: Branin and Himmelblau are relatively smooth with few local minima, Rastrigin and Ackley are highly multimodal with dense local structure, and Levy exhibits ridge-like features.}
  \label{fig:benchmark_landscapes}
\end{figure}

Note that while the theoretical results derived in this paper hold for any dimension $d$, the experimental verification is limited to $d = 2$.

%%%%%%%%%%%%%%%%%%%%%%%%%%%%%%%%%%
\subsubsection{SVR Model Construction}\label{sec:svr-construction}
The aim of this paper is not to improve the prediction or generalization performance of the model, but to analyze its behavior when the learned prediction function is used as the objective in optimization.
Therefore, we systematically select the SVR hyperparameters not to optimize prediction performance but to obtain a diverse set of function shapes.
For these reasons, we do not perform cross-validation for generalization assessment or use a test set for performance evaluation.

In SVR training with an RBF kernel, we specify the hyperparameter vector $\bm{\theta} = (C, \varepsilon, \gamma)^{\top}$.
Since the hyperparameters substantially affect the structure of the prediction function, it is necessary to analyze a wide range of hyperparameter values.
Therefore, we discretize the hyperparameter space $\varTheta$ to obtain a grid $\varTheta_{\mathrm{grid}} \subset \varTheta$.
Furthermore, since the noise level $\tau$ defined in~\cref{sec:dataset} affects the number of support vectors $n_{\mathrm{SV}}$, we treat it as one of the hyperparameters.
Therefore, the set of hyperparameters considered in this paper is $\mathcal{T} \times \varTheta_{\mathrm{grid}}$, and we train an SVR model for each $(\tau, \bm{\theta}) \in \mathcal{T} \times \varTheta_{\mathrm{grid}}$.
Note that before training, each input dimension is standardized to have mean 0 and variance 1, while the target values are kept in their original scale.
All subsequent SVR training, prediction, and DCA are performed in this standardized space.
We denote the collection of trained models by $\mathcal{M} = \{m(\tau, \bm{\theta}) \mid (\tau, \bm{\theta}) \in \mathcal{T} \times \varTheta_{\mathrm{grid}}\}$.

For each model, we record the quantity $C_{\alpha} = \sum_{i \in S_+} \alpha_i + \sum_{i\in S_-} \alpha^*_i$ that appears in the DCA subproblem, along with $\gamma$ and $n_{\mathrm{SV}}$.
The trained models may be unevenly distributed across the values of these metrics.
To address this, for each per-$\tau$ model set $\mathcal{M}_{\tau} = \{m(\tau, \bm{\theta}) \mid \bm{\theta} \in \varTheta_{\mathrm{grid}}\}$, we perform stratified sampling in the space of $(\log_{10} C_{\alpha}, \log_{10} \gamma, \log_{10} n_{\mathrm{SV}})$ to obtain a representative subset $\mathcal{M}^{\mathrm{rep}}_{\tau}$.
Specifically, each of the three log-scaled metrics is divided into $l$ strata by quantiles, and from each non-empty stratum we extract the $k$ models whose Euclidean distance to the stratum center, defined as the midpoint of its bin edges, is smallest.
Working in log scale places these metrics, which span several orders of magnitude, on a common scale, so that this selection is not dominated by any single axis.
If quantile ties leave a stratum empty, it is skipped.

Under this setup, the specific parameter values are as follows.
The hyperparameter grid is the Cartesian product of
\begin{align*}
  C      &\in \{10^{s} \mid s \in [-1,\, 2],\; \text{10 equally spaced}\}, \\
  \varepsilon &\in \{0.01,\, 0.1,\, 1.0\}, \\
  \gamma &\in \{10^{s} \mid s \in [-2,\, 1],\; \text{20 equally spaced}\},
\end{align*}
totaling $10 \times 3 \times 20 = 600$ combinations.
For stratified sampling, we use $l = 6$ strata per axis and $k = 2$ models per stratum.

%%%%%%%%%%%%%%%%%%%%%%%%%%%%%%%%%%
\subsubsection{DCA Implementation}
In DCA, we solve the subproblem~\cref{eq:dca_subprob} based on the current iterate $\bm{x}^{(t)}$ to obtain the next iterate $\bm{x}^{(t+1)}$.
In this paper, since the input of the prediction function is defined on $[0,1]^d$, we impose the box constraint $\bm{x} \in [0,1]^d$.
That is, the subproblem solved at each iteration is given by
\begin{align*}
  \bm{x}^{(t+1)} \in \argmin_{\bm{x}\in [0,1]^d} F^{(t)}(\bm{x}).
\end{align*}
Note that the above formulation is stated in the original input space $\bm{x} \in [0,1]^d$.
Owing to the standardization described in~\cref{sec:svr-construction}, the numerical computation is carried out equivalently in the standardized space, with the box constraint transformed into the corresponding interval constraint.

For the numerical computation, we use SciPy~\cite{Virtanen2020} and adopt L-BFGS-B for solving the subproblem.
Since L-BFGS-B is a gradient-based method, it requires the gradient of the objective function.
The function $F^{(t)}$ used here is analytically differentiable, and its gradient is given in closed form as
\begin{align*}
  \nabla F^{(t)}(\bm{x}) = C_{\alpha}\rho \bm{x} - \bm{s}^{(t)}
  + 2\gamma\sum_{i\in S_-} \alpha^*_i (\bm{x}-\bm{x}_i)
  \exp(-\gamma \lVert \bm{x}-\bm{x}_i\rVert^2).
\end{align*}
Therefore, in the experiments, we use this analytical gradient instead of numerical differentiation.
When $\rho > \rho_{\min}$, the subproblem is strongly convex by~\cref{thm:mu}, and thus has a unique global optimum.
In the numerical computation, we approximate this unique optimum using L-BFGS-B.
That is, the present implementation corresponds to an approximation of the exact-subproblem case considered in~\cref{sec:convergence}.
Henceforth, we proceed under the assumption that each subproblem is solved with sufficient accuracy, and interpret the experimental results as empirical observations of the theoretical results in~\cref{sec:convergence}.
For the L-BFGS-B parameter settings, we use the SciPy~\cite{Virtanen2020} defaults, except for the tolerance, which we set to $\mathrm{ftol} = 10^{-6}$.

For the convergence criterion of DCA, we use both the change in the iterate and the change in the objective value, declaring convergence when
\begin{align*}
    \lVert \bm{x}^{(t)} - \bm{x}^{(t-1)}\rVert \leq \varepsilon_x
    \quad \text{and} \quad
    \lvert \hat{f}(\bm{x}^{(t)}) - \hat{f}(\bm{x}^{(t-1)}) \rvert \leq \varepsilon_f
\end{align*}
hold simultaneously.
Specifically, we set $\varepsilon_x = \varepsilon_f = 10^{-6}$.
We set the maximum number of iterations to $T = 10^4$ and treat the case $t = T$ as a non-convergent run.

In this experiment, we set the DC decomposition parameter to
\begin{align}
  \rho = (1 + 10^{-10}) \cdot \rho_{\min}. \label{eq:rho_setting}
\end{align}
This is approximately equal to the critical value for the convexity guarantee in~\cref{eq:rho_min}.
Under this setting, the relative influence of the quadratic proximal term is minimized and the kernel sum term in the subproblem becomes relatively more dominant.
However, the coefficient $\mu = C_{\alpha}\{\rho - 2\gamma\exp(-\tfrac{3}{2})\}$ in the descent inequality~\cref{eq:descent} becomes extremely small.
Therefore, as long as $\mu > 0$, the monotonic decrease of the objective value is guaranteed under the above assumption, but the quantitative lower bound on the descent provided by~\cref{eq:descent} is close to trivial.
If $\rho$ is set larger, $\mu$ increases and the descent guarantee becomes stronger; however, the quadratic proximal term becomes dominant, suppressing the step taken at each iteration.
Thus, the choice of $\rho$ involves a trade-off between the strength of the descent guarantee and the step size, and the optimal setting is left for future investigation.

%%%%%%%%%%%%%%%%%%%%%%%%%%%%%%%%%%
\subsubsection{Initial Point Selection}
Since DCA is a local optimization algorithm, the choice of the initial point can affect both the convergence behavior and the solution obtained.
The aim of this paper is to evaluate the relationship between the shape of the learned SVR prediction function, which is governed by $C_\alpha$ and $\gamma$, and the convergence behavior of DCA.
Therefore, if the initial point selection depended on the bias of training samples or the shape of the true function, effects from the initial points would contaminate the evaluation, making the interpretation of parameter dependence difficult.
To avoid this issue, we fix the initial points based solely on the search domain $[0,1]^2$, using the four quartile grid points given by the Cartesian product of $0.25$ and $0.75$ in each dimension together with the center $(0.5, 0.5)$, for a total of five points.
Specifically, the set of initial points is given by
\begin{align*}
  \mathcal{X}_0 =
  \{(0.5,0.5)\}\cup\{(q_1,q_2)\mid q_1,q_2 \in\{0.25,0.75\}\}.
\end{align*}

The subproblem~\cref{eq:dca-subprob3} consists of a proximal quadratic term centered at $\bm{v}^{(t)}$, defined by~\cref{eq:v_def}, and the kernel sum $-\sum_{i\in S_-} \alpha^*_i \phi_i(\bm{x})$ over the sample set $S_-$.
If we consider only the proximal quadratic term, the optimum coincides with its center $\bm{v}^{(t)}$, while the kernel sum displaces this center.
Since this displacement strongly depends on the number of training samples and the SVR hyperparameters, we adopt $\bm{v}^{(t)}$ as the initial point for the subproblem search.

%%%%%%%%%%%%%%%%%%%%%%%%%%%%%%%%%%%%%%%%%%
\subsection{Evaluation Metrics}\label{sec:metrics}
In this experiment, we define metrics to evaluate the convergence properties of DCA.
For each model $m \in \mathcal{M}^{\mathrm{rep}}_{\tau}$ and each initial point $\bm{x}_0 \in \mathcal{X}_0$, we use the number of iterations $n_{\mathrm{iter}}(\bm{x}_0; m)$ required for DCA to meet the convergence criterion as a measure of the convergence properties.
Note that if the maximum number of iterations $T$ is reached, we set $n_{\mathrm{iter}}(\bm{x}_0; m) = T$.

Moreover, since DCA is a local optimization algorithm, the solution reached and the convergence properties may differ depending on the initial point selection.
Therefore, as a metric to evaluate the convergence difficulty of DCA for model $m$, we use the median number of iterations over the initial points,
\begin{align*}
    \bar{n}_{\mathrm{iter}}(m) = \median_{\bm{x}_0 \in \mathcal{X}_0}\, n_{\mathrm{iter}}(\bm{x}_0; m).
\end{align*}
Furthermore, as a metric of the initial-point dependence for model $m$, we define the range of the convergence property over the $|\mathcal{X}_0|$ initial points as
\begin{align*}
    R(m) = \max_{\bm{x}_0 \in \mathcal{X}_0} n_{\mathrm{iter}}(\bm{x}_0; m) - \min_{\bm{x}_0 \in \mathcal{X}_0} n_{\mathrm{iter}}(\bm{x}_0; m).
\end{align*}
A larger $R(m)$ indicates that the convergence behavior of DCA for model $m$ depends more strongly on the initial point.

When comparing convergence curves, a direct comparison is difficult because the scale of the objective value differs across models.
Therefore, for the sequence $\{\hat{f}(\bm{x}^{(t)})\}_{t=0}^{T}$ obtained from each initial point, we define the normalized residual as
\begin{align}
    \delta^{(t)}(\bm{x}_0; m) = \frac{\hat{f}(\bm{x}^{(t)}) - \hat{f}(\bm{x}^{(T)}(\bm{x}_0))}{\hat{f}(\bm{x}^{(0)}) - \hat{f}(\bm{x}^{(T)}(\bm{x}_0))}.
    \label{eq:normalized_residual}
\end{align}
Here, the reference value $\hat{f}(\bm{x}^{(T)}(\bm{x}_0))$ in the numerator and denominator is the final value reached from that initial point itself.
By this definition, $\delta^{(0)} = 1$ and $\delta^{(T)} = 0$, allowing the convergence process to be compared across different models on the $[0, 1]$ scale.
We use the final value $\hat{f}(\bm{x}^{(T)}(\bm{x}_0))$ because DCA is a local optimization algorithm and may converge to different local solutions depending on the initial point.
Therefore, even when the local solutions reached differ, the convergence process for each initial point can be compared on a common scale.

In the experiments, we compare the convergence behavior across models using the median number of iterations $\bar{n}_{\mathrm{iter}}(m)$, the initial-point dependence $R(m)$, and the normalized residual $\delta^{(t)}(\bm{x}_0; m)$.

%%%%%%%%%%%%%%%%%%%%%%%%%%%%%%%%%%%%%%%%%%
\subsection{Results and Discussion}\label{sec:results}
In our experimental setting, $183$ out of a total of $30{,}510$ instances ($0.60\%$) reached the maximum iteration count $T$.
These are included in the subsequent analysis with $n_{\mathrm{iter}} = T$, but since the proportion is sufficiently small, their impact on the statistical analysis is limited.

\subsubsection{Convergence Behavior and $C_{\alpha}\rho$}
We first verify that $C_\alpha\rho$ is the primary single quantity accounting for the convergence behavior of DCA.
\Cref{fig:median_vs_Calp_rho} shows a scatter plot of $\bar{n}_{\mathrm{iter}}(m)$ for each benchmark function.
To quantitatively capture the trend, we divided the horizontal axis $\log_{10} C_\alpha\rho$ into 15 bins with an equal number of samples and computed the median and interquartile range for each bin.
Note that bins containing fewer than 5 samples were excluded from the aggregation.

From~\cref{fig:median_vs_Calp_rho}, we observe that in all six functions, $\bar{n}_{\mathrm{iter}}(m)$ tends to increase approximately monotonically as $C_\alpha\rho$ increases.
Furthermore, the markers for different noise levels $\tau$ broadly overlap along the same trend and show no systematic separation by $\tau$.
That is, although the noise level changes the number of support vectors and the coefficient structure, its effect appears to be aggregated into $C_\alpha\rho$ rather than serving as an independent explanatory variable for $\bar{n}_{\mathrm{iter}}(m)$.
The Spearman rank correlation shown later in~\cref{tab:spearman} confirms this quantitatively.
This observation is consistent with the role of $C_\alpha\rho$ as the aggregate curvature scale of the subproblem, established in~\cref{sec:convergence}.

\Cref{fig:range_vs_Calp_rho} shows a scatter plot with $R(m)$ on the vertical axis.
From this plot, we observe that the larger $C_\alpha\rho$ is, the greater the variability in the number of iterations across initial points.
This indicates that the differences in step size due to the initial point also tend to accumulate, so that $R(m)$ likewise increases.

These results show that $C_\alpha\rho$ largely accounts for both the median number of iterations $\bar{n}_{\mathrm{iter}}(m)$ and the initial-point dependence $R(m)$ of DCA.

\begin{figure}[tb]
  \centering
  \includegraphics[width=\textwidth]{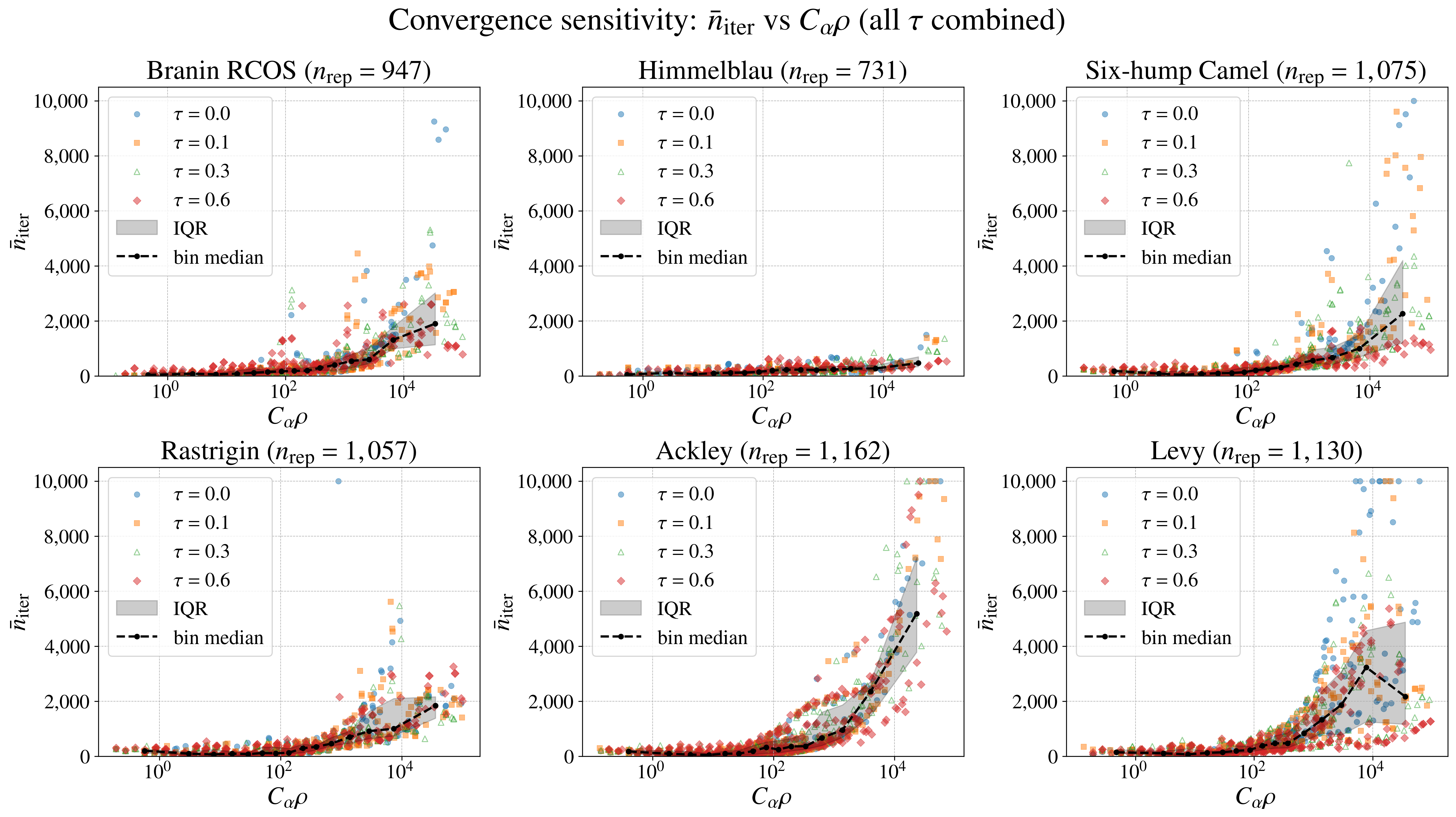}
  \caption{Scatter plot of $\bar{n}_{\mathrm{iter}}(m)$ vs.\ $C_{\alpha}\rho$ for all $\tau \in \mathcal{T}$.
  The dashed line indicates the bin median, and the shaded band indicates the interquartile range.
  In each panel title, $n_{\mathrm{rep}}$ denotes the total number of representative models over all $\tau \in \mathcal{T}$.}
  \label{fig:median_vs_Calp_rho}
\end{figure}

\begin{figure}[tb]
  \centering
  \includegraphics[width=\textwidth]{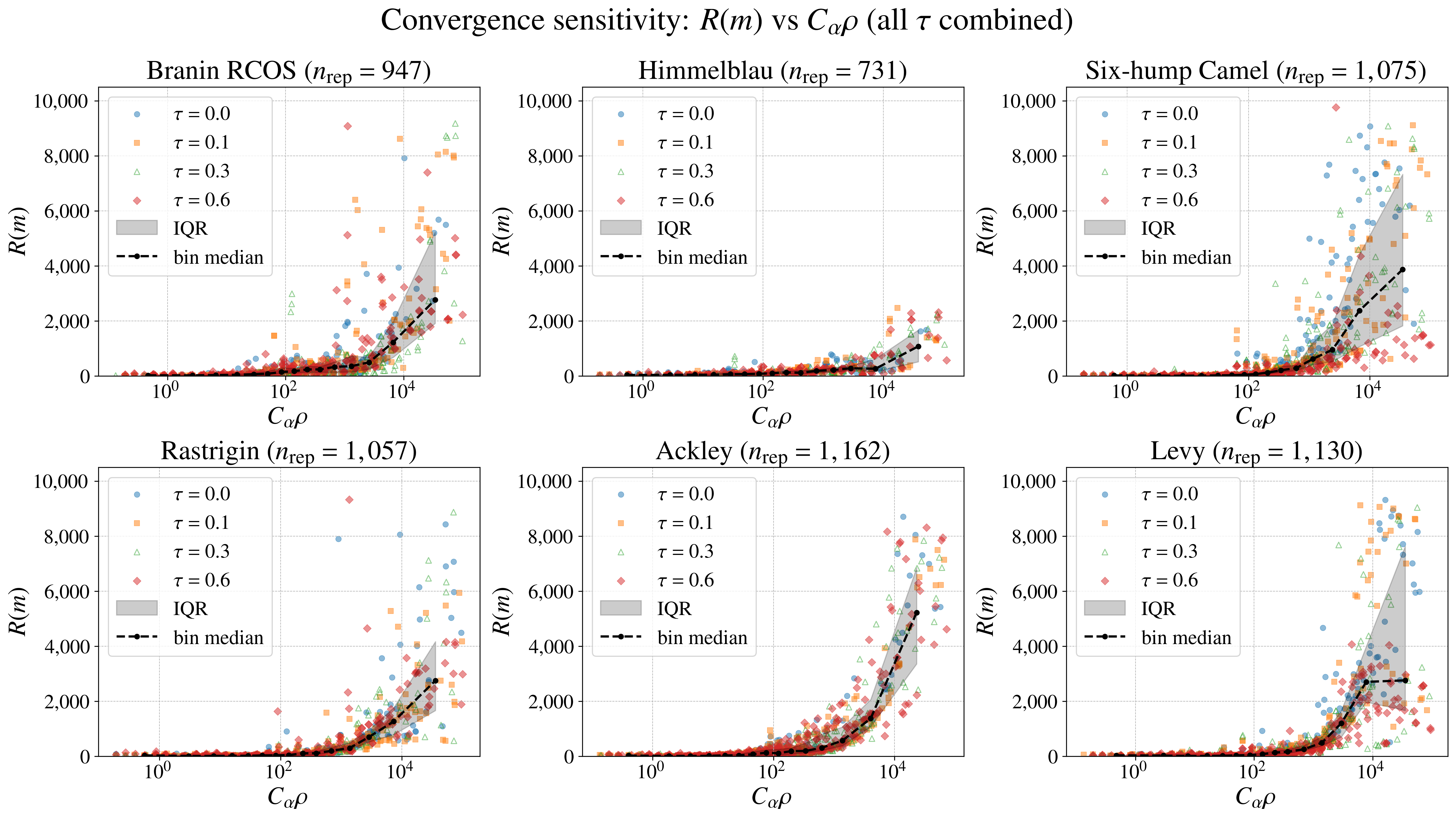}
  \caption{Scatter plot of $R(m)$ vs.\ $C_{\alpha}\rho$ for all $\tau \in \mathcal{T}$.
  Plotting conventions are as in~\cref{fig:median_vs_Calp_rho}.}
  \label{fig:range_vs_Calp_rho}
\end{figure}

%%%%%%%%%%%%%%%%%%%%%%%%%%%%%%%%%%
\subsubsection{Convergence Behavior in the $(C, \gamma)$ Space}
Next, we examine the distribution of $\bar{n}_{\mathrm{iter}}(m)$ over the space of the SVR hyperparameters $(C, \gamma)$, which directly affect $C_\alpha\rho$.
\Cref{fig:heatmap} shows a heatmap of the median of $\bar{n}_{\mathrm{iter}}(m)$ for each combination of $(C, \gamma)$.
Since~\cref{fig:median_vs_Calp_rho} indicated that $C_\alpha\rho$ largely accounts for $\bar{n}_{\mathrm{iter}}(m)$ and that $\tau$ has no independent effect, we aggregate all noise levels $\tau \in \mathcal{T}$ in this analysis.

In all six functions, $\bar{n}_{\mathrm{iter}}(m)$ tends to increase in the region where both $C$ and $\gamma$ are large.
This can be understood as a dual contribution to $C_\alpha\rho$: an increase in $C$ raises the upper bound of the support vector coefficients $\alpha_i, \alpha_i^*$ and thereby increases $C_\alpha$, while an increase in $\gamma$ directly raises the lower bound of $\rho$ in~\cref{eq:rho_setting}.

On the other hand, when only one of $C$ and $\gamma$ is large, $\bar{n}_{\mathrm{iter}}(m)$ does not necessarily increase.
This can be explained by the training structure of SVR.
When $C$ is large but $\gamma$ is small, the kernel is wide, so the prediction function is smooth; even if the upper bound of the coefficients is high, cancellation among support vectors readily occurs, and $C_\alpha$ does not easily grow to its upper bound.
Conversely, when $\gamma$ is large but $C$ is small, the kernel is narrow and the prediction function becomes complex, but since each coefficient is capped at $C$, $C_\alpha$ itself is constrained.
As a result, $C_\alpha\rho$ becomes large only in the region where both $C$ and $\gamma$ are large, and the concentration in the upper-right region seen in~\cref{fig:heatmap} reflects this structure.
However, in the two-axis representation of~\cref{fig:heatmap}, model groups with similar $C_\alpha\rho$ but different combinations of $(C, \gamma)$ are dispersed across different cells, so the effect of $C_\alpha\rho$ cannot be directly isolated.
\begin{figure}[tb]
  \centering
  \includegraphics[width=\textwidth]{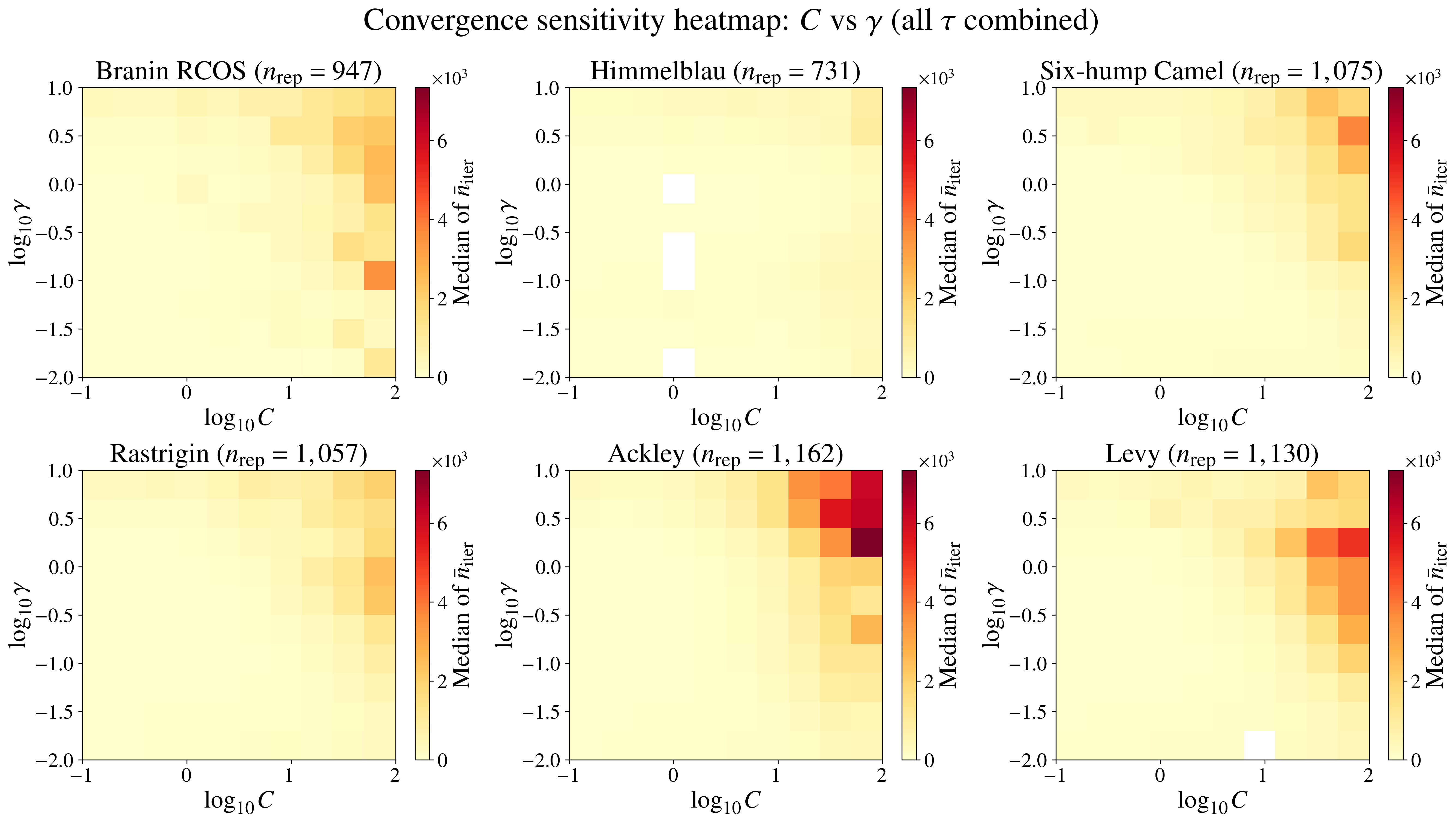}
  \caption{Median of $\bar{n}_{\mathrm{iter}}(m)$ in the $(C, \gamma)$ space for all $\tau \in \mathcal{T}$.
  As in~\cref{fig:median_vs_Calp_rho}, $n_{\mathrm{rep}}$ in each panel title is the total number of representative models over all $\tau \in \mathcal{T}$.}
  \label{fig:heatmap}
\end{figure}

%%%%%%%%%%%%%%%%%%%%%%%%%%%%%%%%%%
\subsubsection{Quantitative Evaluation via Spearman Rank Correlation}
To quantitatively evaluate the relationship between $C_\alpha\rho$ and $\bar{n}_{\mathrm{iter}}(m)$, \cref{tab:spearman} presents the Spearman rank correlation coefficient $r_s$ between $\bar{n}_{\mathrm{iter}}(m)$ and each of $C_\alpha\rho$, $C$, $\gamma$, $\varepsilon$, and $\tau$.
For all six functions, $r_s(C_\alpha\rho)$ lies in the range $0.633$--$0.835$, indicating that $C_\alpha\rho$ is consistently strongly correlated with $\bar{n}_{\mathrm{iter}}(m)$.
In addition, $r_s(C)$ lies in the range $0.557$--$0.832$ and approaches $r_s(C_\alpha\rho)$ for some functions.
For example, Levy gives $r_s(C_\alpha\rho) = 0.835$ while $r_s(C) = 0.832$, and for Ackley $r_s(C) = 0.797$ slightly exceeds $r_s(C_\alpha\rho) = 0.792$.
This stems from the strong $C \to C_\alpha$ pathway shown later in~\cref{tab:Calpha_corr}, under which $C$ alone already captures much of the variation in $\bar{n}_{\mathrm{iter}}(m)$.
In contrast, $r_s(\gamma)$ is relatively weak, ranging from $0.179$ to $0.479$.
This is because $\gamma$ acts on $C_\alpha\rho$ only through $\rho$, and its contribution to the number of iterations is smaller than that of the $C \to C_\alpha$ pathway.
Furthermore, $r_s(\varepsilon)$ is $-0.158$--$0.018$ and $r_s(\tau)$ is $-0.096$--$0.130$, both essentially zero across all functions, which quantitatively confirms that the hyperparameter $\varepsilon$ and the noise level $\tau$ do not serve as independent explanatory variables for $\bar{n}_{\mathrm{iter}}(m)$.
This is consistent with the visual observation in~\cref{fig:median_vs_Calp_rho} and confirms that even if $\tau$ and $\varepsilon$ affect the coefficients of the prediction function through the support vector structure, their effect is aggregated into the composite quantity $C_\alpha\rho$.
Taken together, these results show that the composite quantity $C_\alpha\rho$ is the primary single index summarizing the convergence difficulty of DCA.

\begin{table}[tb]
  \centering
  \caption{Spearman rank correlation coefficients with $\bar{n}_{\mathrm{iter}}(m)$. 
           The column $n_{\mathrm{rep}}$ gives the number of representative models. 
           Here, $r_s(C_{\alpha}\rho)$, $r_s(C)$, and $r_s(\gamma)$ satisfied $p < 0.001$ for all functions.
           The absolute values of $r_s(\varepsilon)$ and $r_s(\tau)$ did not exceed $0.158$, negligible compared with $C_{\alpha}\rho$.}
  \label{tab:spearman}
  \begin{tabular}{lrrrrrr}
    \toprule
    Function & $n_{\mathrm{rep}}$ & $r_s(C_{\alpha}\rho)$ & $r_s(C)$ & $r_s(\gamma)$ & $r_s(\varepsilon)$ & $r_s(\tau)$ \\
    \midrule
    Branin         &  947    & 0.822 & 0.662 & 0.479 & $-$0.094 & 0.119 \\
    Himmelblau     &  731    & 0.633 & 0.557 & 0.347 & 0.018 & 0.051 \\
    Six-hump Camel & 1{,}075 & 0.811 & 0.685 & 0.376 & $-$0.096 & 0.018 \\
    Rastrigin      & 1{,}057 & 0.777 & 0.754 & 0.283 & $-$0.035 & 0.015 \\
    Ackley         & 1{,}162 & 0.792 & 0.797 & 0.179 & $-$0.158 & 0.130 \\
    Levy           & 1{,}130 & 0.835 & 0.832 & 0.282 & $-$0.086 & $-$0.096 \\
    \bottomrule
  \end{tabular}
\end{table}

Next, to examine how well the post-training quantity $C_\alpha\rho$ can be estimated before training, \cref{tab:Calpha_corr} presents the Spearman rank correlation coefficients of $C$ and $\gamma$ with $C_\alpha$.
Although $C_\alpha\rho$ is, strictly speaking, a quantity determined by the dual coefficients obtained after training the SVR, its structure decomposes into two independent pathways through the SVR hyperparameters $(C, \gamma)$.

First, $C_\alpha$ is almost completely governed by $C$.
As shown in~\cref{tab:Calpha_corr}, $r_s(C, C_\alpha) \geq 0.955$ for all six functions, and notably reaches $0.990$ or higher for five of them.
This stems from the fact that, in the dual structure of SVR, the bound support vectors satisfy $\alpha_i = C$ or $\alpha_i^* = C$, so that the scaling of $C$ carries over directly to $C_\alpha$.
Second, $\rho$ is directly proportional to $\gamma$ through its lower bound~\cref{eq:rho_min} and the experimental setting~\cref{eq:rho_setting}.
Therefore, $\gamma$ acts on $C_\alpha\rho$ through $\rho$.
In contrast, as shown in~\cref{tab:Calpha_corr}, the absolute value of $r_s(\gamma, C_\alpha)$ remains at most $0.112$ for all six functions, indicating that $\gamma$ has essentially no effect on $C_\alpha$ itself.

\begin{table}[tb]
  \centering
  \caption{Spearman rank correlation coefficients of $C$ and $\gamma$ with $C_{\alpha}$.
           While $C$ almost completely governs $C_{\alpha}$, $\gamma$ has almost no effect on $C_{\alpha}$.}
  \label{tab:Calpha_corr}
  \begin{tabular}{lrr}
    \toprule
    Function & $r_s(C, C_{\alpha})$ & $r_s(\gamma, C_{\alpha})$ \\
    \midrule
    Branin         & 0.990 & $-$0.076 \\
    Himmelblau     & 0.991 & $-$0.066 \\
    Six-hump Camel & 0.991 & $-$0.112 \\
    Rastrigin      & 0.992 & $-$0.092 \\
    Ackley         & 0.955 & $-$0.090 \\
    Levy           & 0.992 & $-$0.071 \\
    \bottomrule
  \end{tabular}
\end{table}

In this way, the composite quantity $C_\alpha\rho$ separates into two independent pathways, $C \to C_\alpha$ and $\gamma \to \rho$, both of which are governed by the pre-training hyperparameters.
Although the exact value of $C_\alpha\rho$ can be computed only after training the SVR, its primary variation is predictable from the pre-training hyperparameters $(C, \gamma)$.
This is also consistent with \cref{tab:spearman}, where $C$ exhibits a moderate-to-strong correlation with $\bar{n}_{\mathrm{iter}}(m)$ and $\gamma$ exhibits a weak-to-moderate one.
That is, even at the model design stage, before the SVR is trained, the convergence properties of DCA can be estimated through the choice of $C$ and $\gamma$.

%%%%%%%%%%%%%%%%%%%%%%%%%%%%%%%%%%
\subsubsection{Convergence Curves by $C_{\alpha}\rho$ Band}
Finally, we directly visualize how the convergence process differs with the magnitude of $C_\alpha\rho$.
We partition all models into three groups by the terciles of $C_\alpha\rho$, that is, the 33.3 and 66.7 percentiles: low (below the lower tercile), mid (between the two terciles), and high (at or above the upper tercile).
\Cref{fig:convergence_curves} shows the evolution of the normalized residual $\delta^{(t)}$ defined in~\cref{eq:normalized_residual} at the center initial point $(0.5, 0.5)$.
Each line indicates the median of the corresponding group, and the shaded band indicates the interquartile range.

In all six functions, the low group converges fastest and the high group slowest.
This ordering is consistent with the theoretical prediction in~\cref{sec:convergence}.
Furthermore, the vertical axis of~\cref{fig:convergence_curves} is logarithmic, and the nearly linear decrease of the median of each group suggests an exponential decay of $\delta^{(t)}$ with the number of iterations, that is, approximately linear convergence.
The descent inequality~\cref{eq:descent} is a guarantee on the amount of decrease in the objective value at each iteration, and does not directly characterize the convergence rate of $\delta^{(t)}$.
In particular, under the $\rho$ setting~\cref{eq:rho_setting} adopted in this experiment, the coefficient $\mu$ of the descent inequality is extremely small, so the lower bound provided by~\cref{eq:descent} is extremely weak.
On the other hand, the approximately linearly convergent behavior empirically observed for $\delta^{(t)}$ suggests that the actual convergence is better than the estimate given by~\cref{eq:descent}.

In addition, the bands are markedly wider for the high group, indicating that the variability of convergence properties across models is large in the region where $C_\alpha\rho$ is high.
This is likely because, in the region where $C_\alpha\rho$ is large, the number of iterations until convergence is large, so that slight structural differences across models accumulate and manifest as the final variability.
The same trend was also confirmed for the other four initial points.

\begin{figure}[tb]
  \centering
  \includegraphics[width=\textwidth]{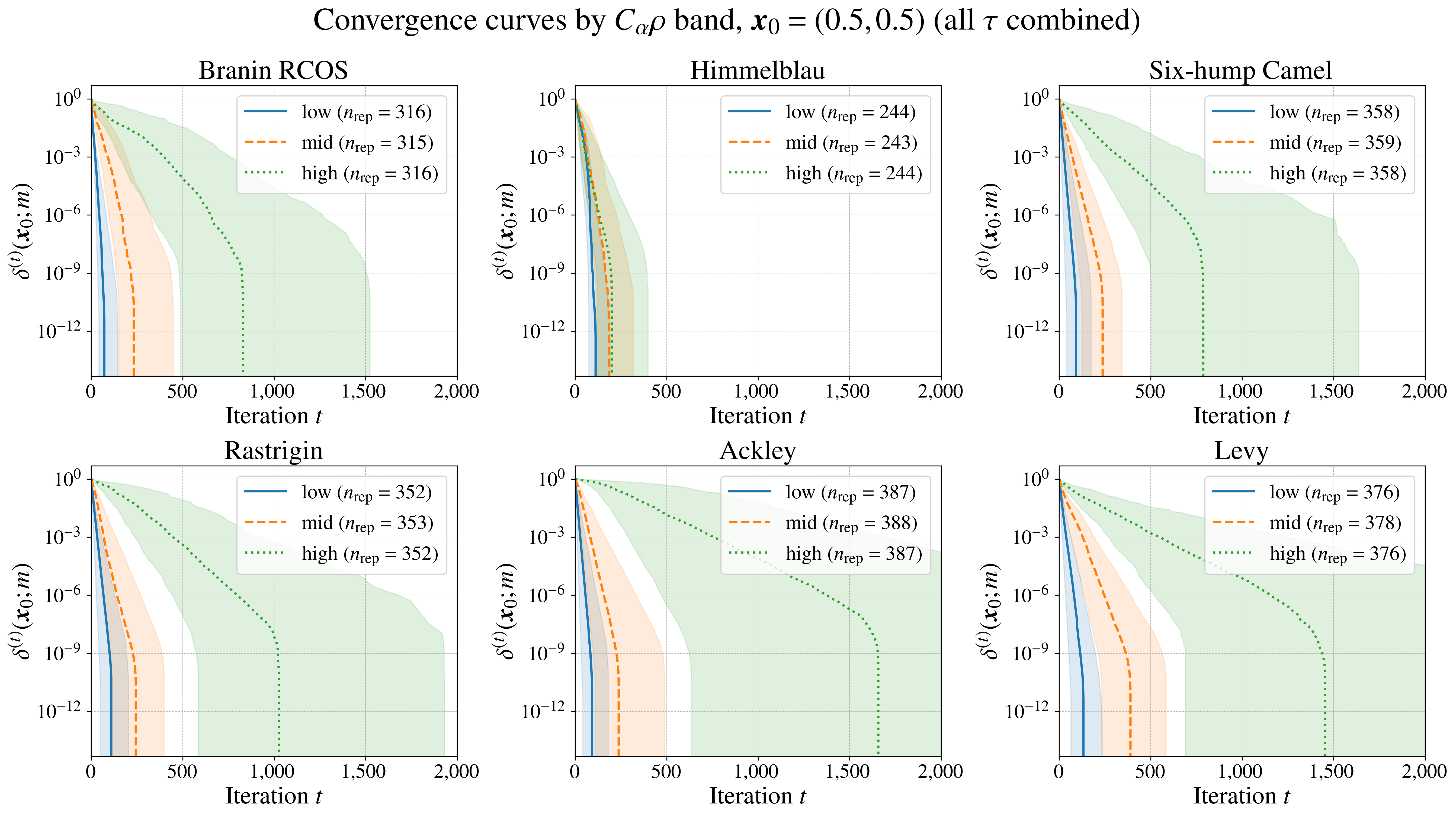}
  \caption{Evolution of the normalized residual $\delta^{(t)}$ by $C_{\alpha}\rho$ band (center initial point). 
           Each line indicates the median, and the shaded band indicates the interquartile range.
           In the legend, $n_{\mathrm{rep}}$ is the number of representative models in each $C_{\alpha}\rho$ group.}
  \label{fig:convergence_curves}
\end{figure}

Across six benchmark functions and four noise levels, these experimental results show that $C_\alpha\rho$ is a single composite index that captures both the convergence speed and the initial-point dependence of DCA.
This is also consistent with the analytical structure shown in~\cref{sec:convergence}, namely that the lower bound $\mu$ of the strong convexity parameter and the upper bound $L$ of the gradient Lipschitz constant of the DC components $G$ and $H$ are both governed by the leading term $C_\alpha\rho$.
That is, we confirm that the theoretical result---that both the strong convexity and the smoothness of the subproblem are characterized by $C_\alpha\rho$---manifests as the experimental observation that $C_\alpha\rho$ largely accounts for the number of iterations and the initial-point dependence of DCA.

The quantity $C_\alpha\rho$ can be computed in closed form from the kernel parameter $\gamma$ and the training result before the DCA iterations begin.
Furthermore, as shown in~\cref{tab:Calpha_corr}, the primary variation of $C_\alpha\rho$ can be estimated from the hyperparameters $(C, \gamma)$ alone, without training the SVR.
Therefore, the difficulty of convergence, that is, whether DCA requires many iterations or is sensitive to the choice of the initial point, can be assessed in advance at two levels: approximately from $(C, \gamma)$ at the design stage before training, and exactly in closed form after training.
This is a practical advantage when designing an RBF-SVR surrogate that is to be optimized by DCA.

%%%%%%%%%%%%%%%%%%%%%%%%%%%%%%%%%%%%%%%%%%
\section{Conclusion}\label{sec:conclusion}
While RBF-SVR has high approximation capability, the resulting prediction function is nonlinear and nonconvex, and its optimization is generally difficult.
In this paper, we analyzed the DC decomposition and the structure of the subproblem when DCA is applied to such a prediction function, and obtained the following results.

Theoretically, by exploiting the structure of the Hessian of the RBF kernel, we derived in closed form the lower bound $\mu$ of the strong convexity parameter and the upper bound $L$ of the gradient Lipschitz constant of the DC components $G$ and $H$.
Both of these analytical constants are governed by the common leading term $C_{\alpha}\rho$.
That is, the assessment of the convergence properties of DCA rests primarily on the single quantity $C_{\alpha}\rho$, which can be evaluated in closed form from the SVR hyperparameters and the training result before the optimization is performed.

Experimentally, we showed on six benchmark functions that $C_{\alpha}\rho$ is the primary single quantity characterizing the convergence properties and the initial-point dependence of DCA.
Furthermore, we showed that $C_{\alpha}\rho$ decomposes into two independent pathways, $C \to C_{\alpha}$ and $\gamma \to \rho$, and that its primary variation is governed by the SVR hyperparameters $(C, \gamma)$.
We also showed that the difficulty of DCA convergence can be assessed in advance at two levels: approximately from $(C, \gamma)$ at the design stage before training, and exactly in closed form after training.

These results imply that $C_{\alpha}\rho$ not only serves as a pre-diagnosis index for DCA but, when the SVR prediction function is designed as a surrogate, can also provide a guideline for selecting hyperparameters that accounts for the optimization difficulty.
In addition, the gradient Lipschitz upper bound $L$ derived in this paper provides an analytical reference value for the smoothness constant required by DC-FW~\cite{Maskan2025,Pokutta2025} when it is applied to the minimization of the RBF-SVR prediction function.
An empirical study of DC-FW on the RBF-SVR prediction function using this analytical constant is left for future work.

\appendix
\section*{Appendix}
\addcontentsline{toc}{section}{Appendix}
\section{Eigenvalue bounds for the Hessian of the Gaussian RBF kernel}

\begin{lemma}[Eigenvalue bounds for the Hessian of the Gaussian RBF kernel] \label{lem:kernel_eig}
Let $\gamma > 0$. For the RBF kernel $\phi_i(\bm{x}) = \exp(-\gamma\lVert\bm{x}-\bm{x}_i\rVert^2)$, the inequalities
\begin{align*}
  \lambda_{\max}(\nabla^2 \phi_i(\bm{x})) &\leq 4\gamma \exp\!\left(-\tfrac{3}{2}\right), \\
  \lambda_{\min}(\nabla^2 \phi_i(\bm{x})) &\geq -2\gamma,
\end{align*}
hold for all $\bm{x} \in \mathbb{R}^d$.
\end{lemma}

\begin{proof}
The Hessian of the RBF kernel $\phi_i$ is given by
\begin{align*}
  \nabla^2\phi_i(\bm{x}) = 4\gamma^2\phi_i(\bm{x})\bm{r}\bm{r}^{\top} - 2\gamma\phi_i(\bm{x})I,
\end{align*}
where $\bm{r} = \bm{x} - \bm{x}_i$.
For any unit vector $\bm{v} \in \mathbb{R}^d$, the Rayleigh quotient is
\begin{align*}
  \bm{v}^{\top}\nabla^2\phi_i(\bm{x})\bm{v}
    = 4\gamma^2\phi_i(\bm{x})\langle \bm{r}, \bm{v} \rangle^2 - 2\gamma\phi_i(\bm{x}).
\end{align*}
From $\lVert\bm{v}\rVert = 1$ and the Cauchy-Schwarz inequality, we have
\begin{align}
  0 \leq \langle \bm{r}, \bm{v} \rangle^2 \leq \lVert\bm{r}\rVert^2. \label{eq:rtv_bounds}
\end{align}
Using~\cref{eq:rtv_bounds}, we will show~\cref{lem:kernel_eig}.

We first show the upper bound on the maximum eigenvalue.
Since $\phi_i(\bm{x}) > 0$, applying the upper bound in~\cref{eq:rtv_bounds} yields
\begin{align*}
  \bm{v}^{\top}\nabla^2\phi_i(\bm{x})\bm{v} \leq \psi(\lVert\bm{r}\rVert^2),
  \quad \text{where} \quad
  \psi(u) = \exp(-\gamma u)(4\gamma^2 u - 2\gamma).
\end{align*}
Since $\psi'(u) = 2\gamma^2(3 - 2\gamma u)\exp(-\gamma u)$, the function $\psi$ attains a local maximum value $\psi(\tfrac{3}{2\gamma}) = 4\gamma\exp(-\tfrac{3}{2})$ at $u = \tfrac{3}{2\gamma}$, which is the maximum on $u \geq 0$.
Therefore, $\bm{v}^{\top}\nabla^2\phi_i(\bm{x})\bm{v} \leq 4\gamma\exp(-\tfrac{3}{2})$ holds for any $\bm{x}$ and any unit vector $\bm{v}$, and by the variational characterization of the maximum eigenvalue,
\begin{align*}
  \lambda_{\max}(\nabla^2\phi_i(\bm{x})) \leq 4\gamma\exp\!\left(-\tfrac{3}{2}\right)
\end{align*}
follows.

Next, we show the lower bound on the minimum eigenvalue.
From the lower bound $\langle \bm{r}, \bm{v} \rangle^2 \geq 0$ in~\cref{eq:rtv_bounds} and $0 < \phi_i(\bm{x}) \leq 1$,
\begin{align*}
  \bm{v}^{\top}\nabla^2\phi_i(\bm{x})\bm{v} \geq -2\gamma\phi_i(\bm{x}) \geq -2\gamma
\end{align*}
holds for any $\bm{x}$ and any unit vector $\bm{v}$.
By the variational characterization of the minimum eigenvalue,
\begin{align*}
  \lambda_{\min}(\nabla^2\phi_i(\bm{x})) \geq -2\gamma
\end{align*}
follows.
\end{proof}

\bibliographystyle{plain}
\bibliography{refs}

\end{document}